\documentclass[sigconf]{acmart}

\AtBeginDocument{%
  \providecommand\BibTeX{{%
    \normalfont B\kern-0.5em{\scshape i\kern-0.25em b}\kern-0.8em\TeX}}}

\usepackage{multirow}
\usepackage{amsmath}
\usepackage{makecell}
\usepackage{multirow}
\usepackage{mathrsfs}
\usepackage[ruled]{algorithm2e}
\usepackage{cellspace}
\usepackage{makecell}
\usepackage{bm}
\usepackage{graphicx}
\usepackage{paralist}
\usepackage{enumitem}
\usepackage[utf8]{inputenc} 
\usepackage[T1]{fontenc}    
\usepackage{hyperref}       
\usepackage{url}            
\usepackage{booktabs}       
\usepackage{amsfonts}       
\usepackage{nicefrac}       
\usepackage{microtype}      
\usepackage[ruled]{algorithm2e}
\usepackage{threeparttable}

\usepackage{graphicx}
\usepackage{subfigure}
\usepackage{booktabs} 
\usepackage{caption}

\usepackage{framed}
\usepackage{amsfonts}
\usepackage{mathrsfs}
\usepackage{mathtools}
\usepackage{array}
\usepackage{amsthm}
\usepackage{verbatim} 
\usepackage{enumerate}
\usepackage{bbm}
\usepackage{commath}
\usepackage{float}
\usepackage{amsmath}
\usepackage{tabularx}
\usepackage{listings}
\usepackage{xcolor}
\usepackage{colortbl}
\usepackage{tcolorbox}
\usepackage{bbding}
\usepackage{multirow}
\usepackage[toc,page,header]{appendix}
\usepackage{minitoc}
\usepackage{titletoc}

\copyrightyear{2023}
\acmYear{2023}
\setcopyright{acmlicensed}
\acmConference[KDD '23] {Proceedings of the 29th ACM SIGKDD Conference on Knowledge Discovery and Data Mining}{August 6--10, 2023}{Long Beach, CA, USA.}
\acmBooktitle{Proceedings of the 29th ACM SIGKDD Conference on Knowledge Discovery and Data Mining (KDD '23), August 6--10, 2023, Long Beach, CA, USA}
\acmPrice{15.00}
\acmISBN{979-8-4007-0103-0/23/08}
\acmDOI{10.1145/3580305.3599263}

\begin{document}

\newtheorem{thm}{Theorem} 
\newtheorem{remark}[thm]{Remark}
\definecolor{comment}{RGB}{70, 150, 60}

\newcommand{\xm}{x^{-}}
\newcommand{\xp}{x^{+}}
\newcommand{\model}{BatchSampler\xspace}
\newcommand{\graph}{proximity graph\xspace}

\newcommand{\Set}[1]{\mathcal{#1}}
\newcommand{\vpara}[1]{\vspace{0.07in}\noindent\textbf{#1 }}
\newcommand{\Mat}[1]{\mathbf{#1}}
\newtheorem{prop}{Proposition}
\newtheorem{problem}{Problem}
\newcommand{\hide}[1] 

\title{\model: Sampling Mini-Batches for Contrastive Learning in Vision, Language, and Graphs}

\author{Zhen Yang}
\affiliation{%
  \institution{Tsinghua University}
}
\email{yangz21@mails.tsinghua.edu.cn}
\authornote{These authors contributed equally to this research.}

\author{Tinglin Huang}
\authornotemark[1]
\affiliation{%
  \institution{Tsinghua University}
  \institution{Yale University}
}
\email{tinglin.huang@yale.edu}
\authornote{This work was done when the author worked at Tsinghua University. }

\author{Ming Ding}
\authornotemark[1]
\affiliation{%
  \institution{Tsinghua University}
}
\email{dm18@mails.tsinghua.edu.cn}

\author{Yuxiao Dong}
\affiliation{%
  \institution{Tsinghua University}
}
\email{yuxiaod@tsinghua.edu.cn}
\authornote{YD and JT are the corresponding authors.}

\author{Rex Ying}
\affiliation{%
  \institution{Yale University}
}
\email{rex.ying@yale.edu}

\author{Yukuo Cen}
\affiliation{%
  \institution{Tsinghua University}
}
\email{cyk20@mails.tsinghua.edu.cn}

\author{Yangliao Geng}
\affiliation{%
  \institution{Tsinghua University}
}
\email{gengyla@mail.tsinghua.edu.cn}

\author{Jie Tang}
\affiliation{%
  \institution{Tsinghua University}
}
\email{jietang@tsinghua.edu.cn}
\authornotemark[3]

\renewcommand{\shortauthors}{Zhen Yang et al.}
\renewcommand{\shorttitle}{\footnotesize \model: Sampling Mini-Batches for Contrastive Learning in Vision, Language, and Graphs}

\begin{abstract}
In-Batch contrastive learning is a state-of-the-art self-supervised method that brings semantically-similar instances close while pushing dissimilar instances apart within a mini-batch. 
Its key to success is the negative sharing strategy, in which every instance serves as a negative for the others within the mini-batch. Recent studies aim to improve performance by sampling hard negatives \textit{within the current mini-batch}, whose quality is bounded by the mini-batch itself. In this work, we propose to improve contrastive learning by sampling mini-batches from the input data. We present \model\footnote{The code is available at \url{https://github.com/THUDM/BatchSampler}} to sample mini-batches of hard-to-distinguish (i.e., hard and true negatives to each other) instances. To make each mini-batch have fewer false negatives, we design the proximity graph of randomly-selected instances. To form the mini-batch, we leverage random walk with restart on the proximity graph to help sample hard-to-distinguish instances. \model is a simple and general technique that can be directly plugged into existing contrastive learning models in vision, language, and graphs. Extensive experiments on datasets of three modalities show that \model can consistently improve the performance of powerful contrastive models, as shown by significant improvements of SimCLR on ImageNet-100, SimCSE on STS (language), and GraphCL and MVGRL on graph datasets.

\end{abstract}



%


%
\maketitle

\section{Introduction} \label{sec:intro}
\begin{figure}[t]
    \centering
    \includegraphics[width=0.48\textwidth]{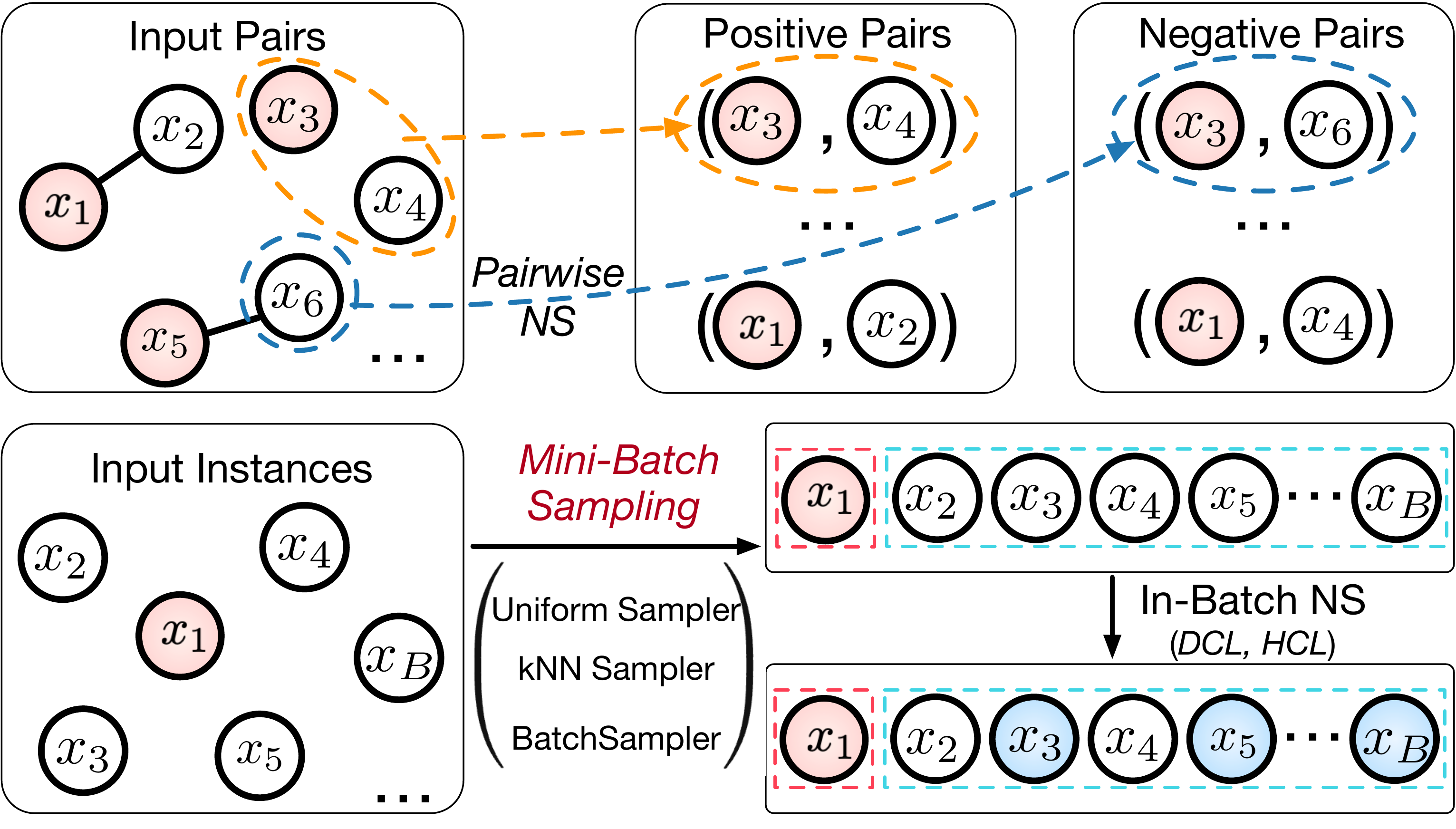}
    \vspace{-4mm}
    \caption{Technical comparison between Mini-Batch Sampling, Pairwise Negative Sampling (NS), and In-Batch Negative Sampling. 
    \textmd{
    The blue nodes are hard negatives. 
    }}
    \label{fig:difference}
    \vspace{-3mm}
\end{figure}

Contrastive learning~\cite{oord2018representation,jaiswal2020survey} is one of the dominant strategies for self-supervised representation learning across various data domains, such as MoCo~\citep{he2019momentum} and SimCLR~\citep{simclr2020} in computer vision, SimCSE~\citep{simcse2021} in natural language processing, and GraphCL~\citep{graphcl2020} in graph representation learning. 
The essence of self-supervised contrastive learning is to make similar instances close to each other and dissimilar ones farther away in the learned representation space.

The course of self-supervised contrastive models usually starts with loading each mini-batch of $B$ instances sequentially from the input data (e.g., the images in Figure \ref{fig:motivating_example} (a)). 
In each batch, each instance $x$ is associated with its augmentation $x^+$ as the positive samples and the other instances as negatives $\{x^-\}$. 
Commonly by using the InfoNCE loss~\cite{oord2018representation}, the goal of contrastive learning is to discriminate instances by mapping positive pairs ($x, x^+$) to similar embeddings and negative pairs ($x, x^-$) to dissimilar embeddings.

Given the self-supervised contrastive setting, the negative samples $\{x^-\}$---with unknown labels---play an essential role in the contrastive optimization process. 
To improve in-batch contrastive learning, there are various attempts to take on them from different perspectives. 
Globally, SimCLR~\cite{simclr2020} shows that simply increasing the size $B$ of the mini-batch (e.g., 8192)---i.e., more negative samples---outperforms previously carefully-designed strategies, such as the memory bank~\citep{wu2018unsupervised}  and the consistency improvement of the stored negatives in MoCo~\citep{he2019momentum}. 
Locally, given each mini-batch, recent studies such as the DCL and HCL~\citep{hard2021,debias2021} methods have focused on identifying true or hard negatives within this batch. 
In other words, existing efforts have largely focused on designing better negative sampling techniques after each mini-batch of instances is loaded.

\begin{figure*}[hbpt]
    \centering
    \includegraphics[width=0.98\textwidth]{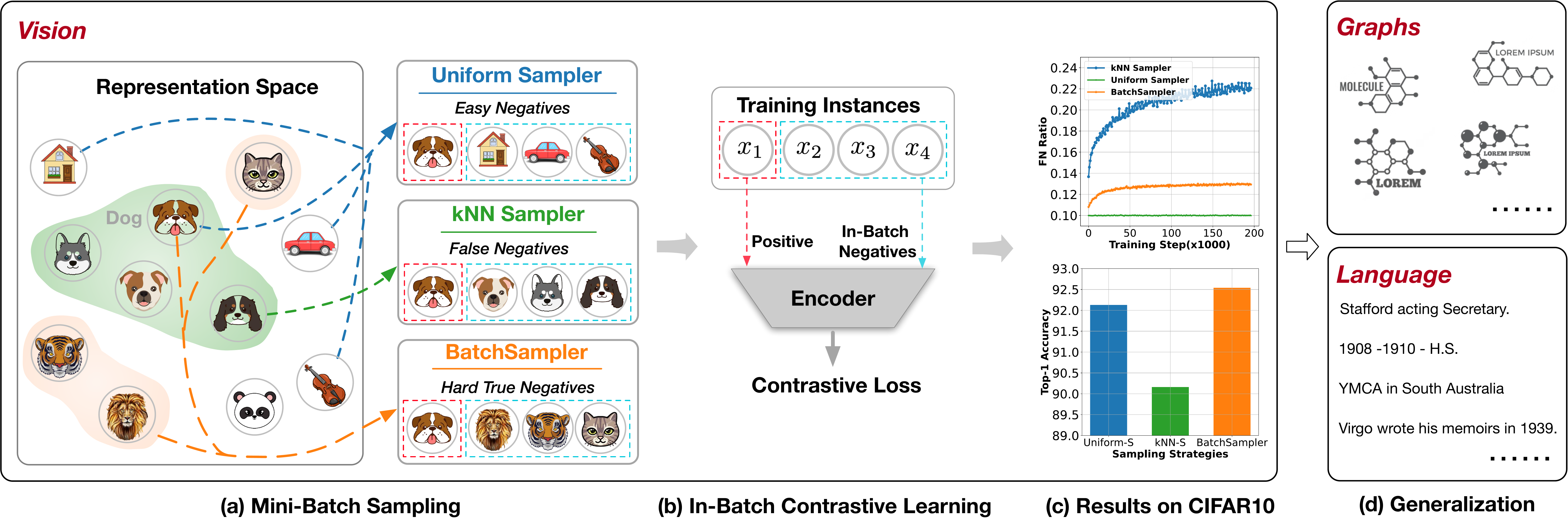}
    \vspace{-2mm}
    \caption{A motivating illustration of \model, using vision as an example. \textmd{Uniform Sampler randomly samples a batch of instances, which contains easy negatives. kNN Sampler selects the nearest instances to form a batch, resulting in so many false negatives. \model samples a mini-batch with hard yet true negatives based on \graph. 
    }}
    \label{fig:motivating_example}
    \vspace{-3mm}
\end{figure*}

\vpara{Problem.}
In this work, we instead propose to globally sample instances from the input data to form mini-batches.  
The goal is to have mini-batches that naturally contain as many hard negatives $\{x^-\}$---that have different (underlying) labels from $x$ (true negatives) but similar representations with $x$---for each instance $x$ as possible. 
In this way, the discrimination between positive and negative instances in each mini-batch can better inform contrastive optimization.

\vpara{Uniform \& kNN Samplers.}
Traditionally, there are two ways to form mini-batches (Cf. Figure ~\ref{fig:motivating_example} (a)). 
The common option is the \textit{Uniform Sampler} that sequentially loads or uniformly samples a batch of instances for each training step\cite{simclr2020,simcse2021,graphcl2020}. 
However, the Uniform Sampler neglects the effect of hard negatives~\citep{mochi2020,hard2021}, and the batches formed contain easy negatives with low gradients that contribute little to optimization~\cite{mcns2020,xiong2020approximate}.  
In order to have hard negatives, it is natural to cluster instances that are nearest to each other in the representation space to form each batch, that is, the \textit{kNN Sampler}. 
Unfortunately, the instances with the same underlying labels are also expected to cluster together~\citep{simclr2020,caron2020unsupervised}, resulting in high percentages of false negatives in the batch (Cf. Figure ~\ref{fig:motivating_example} (c)).

\vpara{Contributions: \model.} 
We present \model to sample mini-batches of hard-to-distinguish instances for in-batch contrastive learning. 
Fundamentally, each mini-batch is required to cover hard yet true negatives, thus addressing the issues faced by Uniform and kNN Samplers. 
To achieve this, we design the proximity graph of randomly-selected instances in which each edge is used to control the pairwise similarity of their representations---that is, the hardness of negative samples. 
The false negative issue in the kNN Sampler is mitigated when random instances are picked to construct the graph. 
To form one batch, we leverage random walk with restart on the proximity graph to draw instances. 
The premise is that the local neighbors sampled by the walkers are similar, that is, hard-to-distinguish.

\model is a simple and general technique that can be directly plugged into in-batch contrastive models in vision, language, and graphs. 
The experimental results show that the well-known contrastive learning models---SimCLR~\citep{simclr2020} and MoCo v3~\citep{chen2021empirical} in vision, SimCSE~\citep{simcse2021} in language, and GraphCL~\citep{graphcl2020} and MVGRL~\cite{hassani2020contrastive} in graphs---can benefit from the mini-batches formed by \model. 
We also theoretically and empirically show that how \model balances the challenges faced in the Uniform and kNN Samplers.

\vpara{Differences from Pairwise Negative Sampling.} 
The problem of mini-batch sampling is different from the pairwise negative sampling problem. 
As shown in Figure~\ref{fig:difference}, the pairwise negative sampling method samples negative instances for each positive pair, whereas the mini-batch sampling methods focus on sampling mini-batches that reuses other instances as negatives in the batch. 
Massive previous works in pairwise negative sampling~\cite{huang2020embedding,xiong2020approximate,karpukhin2020dense,mcns2020,mixgcf2021} have attempted to improve performance by globally selecting similar negatives for a given anchor from the entire dataset. 
However, our focus is on globally sampling mini-batches that contain many hard negatives rather than only selecting hard negatives for each pair.

\section{Related Work}
\label{sec:related_work}
\vpara{Contrastive learning in different modalities.} 
Contrastive learning follows a similar paradigm that contrasts similar and dissimilar observations based on noise contrastive estimation~(NCE)~\citep{nce2010,oord2018representation}. 
The primary distinction between contrastive methods of different modalities is how they augment the data.
As for computer vision, 
MoCo~\citep{he2019momentum} and SimCLR~\citep{simclr2020} augment data with geometric transformation and appearance transformation.
Besides simply using data augmentation, WCL~\citep{zheng2021weakly} additionally utilizes an affinity graph to construct positive pairs for each example within the mini-batch. \citet{wu2021enabling} design a data replacement strategy to mine hard positive instances for contrastive learning. 
As for language, CLEAR~\citep{wu2020clear} and COCO-LM~\citep{meng2021coco} augment the text data through word deletion, reordering, and substitution, 
while SimCSE~\citep{simcse2021} obtains the augmented instances by applying the standard dropout twice.
As for graphs,
DGI~\citep{petar2018deep} and InfoGraph~\citep{sun2019infograph}
treat the node representations and corresponding graph representations as positive pairs. 
Besides, InfoGCL\citep{xu2021infogcl}, JOAO~\citep{you2021graph}, GCA~\cite{zhu2021graph}, GCC~\citep{gcc2020} and GraphCL~\citep{graphcl2020} augment the graph data by graph sampling
or proximity-oriented methods.
MVGRL~\cite{hassani2020contrastive} proposes to compare the node representation in one view with the graph representation in the other view.
\citet{zhu2021empirical} compares different kinds of graph augmentation strategies.
Our proposed \model is a general mini-batch sampler that can directly be applied to any in-batch contrastive learning framework with different modalities.

\vpara{Negative sampling in contrastive learning.} 
Previous studies about negative sampling in contrastive learning roughly fall into two categories: 
\textbf{(1) Memory-based negative sampling strategy}, such as MoCo~\citep{he2019momentum},
maintains a fixed-size memory bank to store negatives which are updated regularly during the training process.
MoCHI~\citep{mochi2020} and m-mix~\cite{zhang2022m} propose to mix the hard negative candidates at the feature level to generate more challenging negative pairs.
MoCoRing~\citep{wu2020conditional} samples hard negatives from a defined conditional distribution which keeps a lower bound on the mutual information.
\textbf{(2) In-batch negative sharing strategy}, such as SimCLR~\citep{simclr2020} and MoCo v3~\citep{chen2021empirical},
adopts different instances in the current mini-batch as negatives.
To mitigate the false negative issue, DCL~\citep{debias2021} modifies the original InfoNCE objective to reweight the contrastive loss. \citet{huynh2022boosting} identifies the false negatives within a mini-batch by comparing the similarity between negatives and the anchor image's multiple support views. 
Additionally, HCL~\citep{hard2021} revises the original InfoNCE objective by assigning higher weights for hard negatives among the mini-batch. Recently, UnReMix~\citep{tabassum2022hard} is proposed to sample hard negatives by effectively capturing aspects of anchor similarity, representativeness, and model uncertainty. However, such locally sampled hard negatives cannot exploit hard negatives sufficiently from the dataset.

Global hard negative sampling methods on triplet loss have been widely investigated, which aim to globally sample hard negatives for a given positive pair.
For example, \citet{wang2021solving} proposes to take rank-k hard negatives from some randomly sampled negatives. \citet{xiong2020approximate} globally samples hard negatives by an asynchronously-updated approximate nearest neighbor (ANN) index for dense text retrieval. 
Different from the above methods which are applied to a triplet loss for a given pair, \model samples mini-batches with hard negatives for InfoNCE loss.

\section{Problem: Mini-Batch Sampling for Contrastive Learning}
\label{sec:prosampler}
\label{sec:ns}

\vpara{In-Batch Contrastive Learning.}
\emph{In-batch contrastive learning} commonly follows or slightly updates the following objective~\citep{nce2010,oord2018representation} across different domains, such as graphs~\cite{graphcl2020,gcc2020}, vision~\cite{simclr2020,chen2021empirical}, and language~\cite{simcse2021}:
\begin{small}
\begin{align}\label{equ:infonce}
\text{min} \ \ \mathbb{E}_{\{x_1...\ x_B\} \subset \mathcal{D}} \left [-\sum_{i=1}^B\log \frac{e^{f(x_i)^T f(x_i^+)}}{e^{f(x_i)^T f(x_i^+) } + \sum_{j\neq i}e^{f(x_i)^T f(x_j)}} \right ],
\end{align}
\end{small}
where $\{x_1...\ x_B\}$ is a mini-batch of samples (usually) sequentially loaded from the dataset $\mathcal{D}$, and $x_i^+$ is an augmented version of $x_i$. The encoder $f(\cdot)$ learns to discriminate instances by mapping different data-augmentation versions of the same instance (positive pairs) to similar embeddings, and mapping different instances in the mini-batch (negative pairs) to dissimilar embeddings.

Usually, the in-batch negative sharing strategy---every instance serves as a negative to the other instances within the mini-batch---is used to boost the training efficiency~\cite{simclr2020}. 
It is then natural to have hard negatives in each mini-batch for improving contrastive learning. 
Straightforwardly, we could attempt to sample hard negatives within the mini-batch~\citep{debias2021,hard2021,karpukhin2020dense}. 
However, the batch size of a mini-batch is---by definition---far smaller than the size of the input dataset, and existing studies show that sampling such a local sampling method fails to effectively explore all the hard negatives~\citep{xiong2020approximate,zhang2013optimizing}.

\vpara{\textit{Problem: Mini-Batch Sampling.}}
The goal of this work is to have a general mini-batch sampling strategy to support different modalities of data. 
Specifically, given a set of data instances $\Set{D}=\{x_1,\cdots,x_N\}$, the objective is to design a modality-independent sampler 
to sample a mini-batch of instances where each pair of instances are hard to distinguish across the dataset.

There are two existing strategies---Uniform Sampler and kNN Sampler---adopted in contrastive learning.  

\vpara{Uniform Sampler} is the most common strategy used in contrastive learning~\citep{simclr2020,simcse2021,graphcl2020}.
The pipeline is to first randomly sample a batch of instances for each training step, then feed them into the model for optimization. 

Though simple and model-independent, Uniform Sampler neglects the effect of hard negatives~\citep{mochi2020,hard2021}, and the batches formed contain negatives with low gradients that contribute little to optimization.  
Empirically, we show that Uniform Sampler results in a low percentage of similar instance pairs in a mini-batch (Cf. Figure~\ref{fig:per_sampler}). 
Theoretically, \citet{mcns2020} and \citet{xiong2020approximate} also prove that the sampled negative should be similar to the query instance since it can provide a meaningful gradient to the model.

\vpara{kNN Sampler} globally samples a mini-batch with many hard negatives. 
As its name indicates, it tries to pick an instance at random and retrieve a set of nearest neighbors to construct a batch.

However, the instances of the same `class' will be clustered together in the embedding space
~\citep{simclr2020,caron2020unsupervised}. 
Hence the negatives retrieved by it are hard at first but would be replaced by false negatives~(FNs) as the training epochs increase, misguiding the model training.

In summary, Uniform Sampler leverages random negatives to guide the optimization of the model; 
whereas the kNN one explicitly samples hard negatives but suffers from false negatives. 
Thus, an ideal negative sampler for in-batch contrastive learning should balance between kNN and Uniform Samplers, ensuring both the exploitation of hard negatives and the mitigation of the FN issue.

\section{The B\texorpdfstring{\MakeLowercase{atch}}{atch}S\texorpdfstring{\MakeLowercase{ampler}}{ampler} Method}

\begin{figure*}[t]
    \centering
    \includegraphics[width=0.98\textwidth]{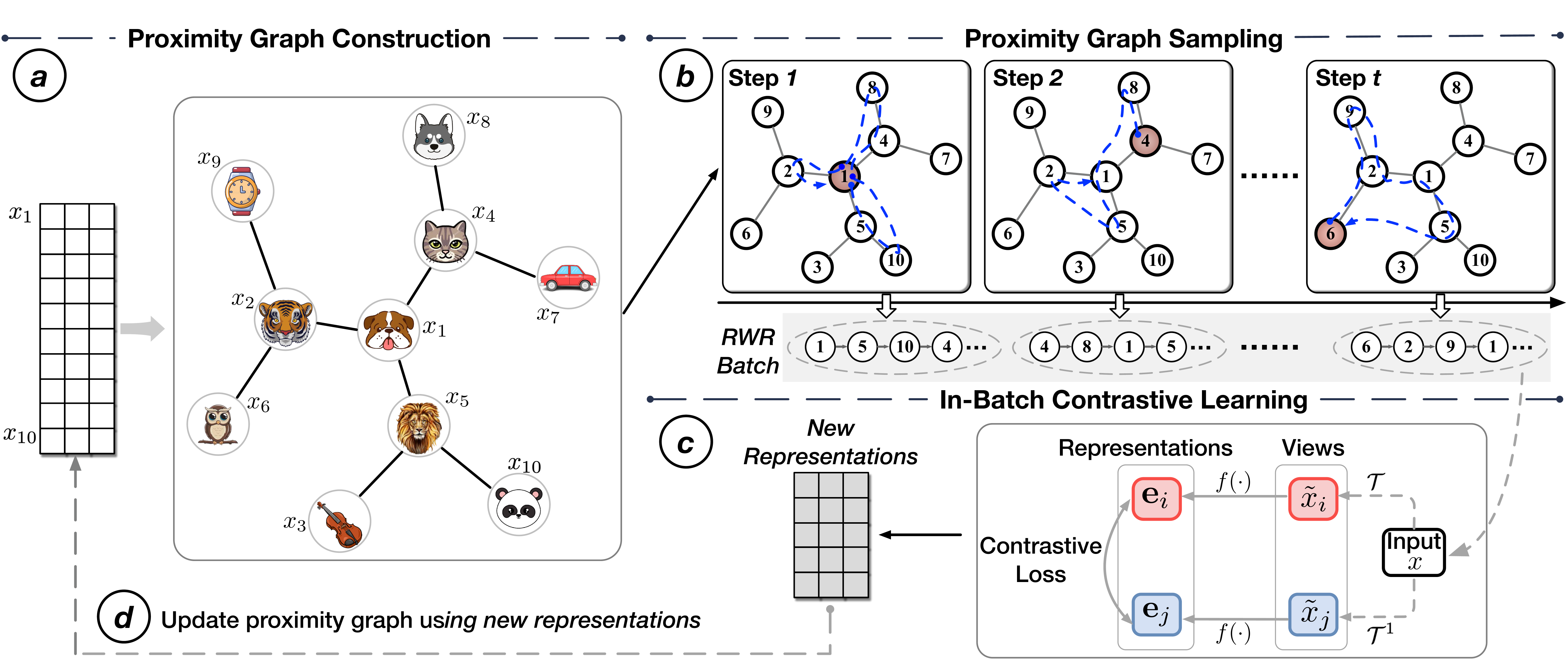}
    \caption{The framework of \model, using the vision modality as an example. \textmd{The \graph is first constructed based on generated image representations and will be updated every $t$ training step. Next, a \graph-based negative sampler is applied to generate a batch with hard negatives for in-batch contrastive learning.}}
    \label{fig:framework}
    \vspace{-2mm}
\end{figure*}

To improve contrastive learning, we propose the \model method with the goal of forming mini-batches with 1) hard-to-distinguish instances while 2) avoiding false negatives. 
\model is a simple and general strategy that can be directly plugged into existing in-batch contrastive learning models in vision, language, and graphs. 
Figure ~\ref{fig:framework} shows the overview of \model.

The basic idea of \model is to form mini-batches by \textbf{globally} sampling instances based on the similarity between each pair of them, which is significantly different from previous local sampling methods in contrastive learning~\cite{debias2021,hard2021,mochi2020}. 
To achieve this, the first step of \model is to construct the proximity graph of instances with edges measuring the pairwise similarity. 
Second, we perform mini-batch sampling as a walk by leveraging a simple and commonly-used graph sampling technique, random walk with restart, to sample instances (nodes) from the proximity graph.

\subsection{Proximity Graph Construction}

The goal of the proximity graph is to collect candidate instances that alleviate the issues faced by Uniform and kNN Sampler by connecting similar instances while reducing the false negative pairs.
If neighbors of an instance are randomly chosen from the dataset, sampling on the resulting graph resembles Uniform Sampler. Conversely, if the most similar instances are chosen as neighbors, it resembles the behavior of kNN Sampler. In \model, we propose a simple strategy to construct the proximity graph as follows.

Given a training dataset, we have $N$ instances $\{v_i|i=1,\cdots,N\}$ and their corresponding representations $\{\Mat{e}_i|i=1,\cdots,N\}$ generated by the encoder $f(\cdot)$. The proximity graph is formulated as:
\begin{align}
    G=(\Set{V},\Set{E}),
    \label{eq_G}
\end{align}
where the node set $\Set{V}=\{v_1,\cdots,v_N\}$ denotes the instances and $\Set{E}\subseteq\{(v_i,v_j)|v_i,v_j\in\Set{V}\}$ is a set of node pairs. 
Let $\Set{N}_{i}$ be the neighbor set of $v_i$ in the \graph. 
To construct $\Set{N}_{i}$, we first form a candidate set $\mathcal{C}_i=\{v_m\}$
for each instance $v_i$ by uniformly picking $M(M\ll N)$ neighbor candidates. 
Each node and its neighbor form an edge, resulting in the edge set $\Set{E}$.
Then we select the $K$ nearest ones from the candidate set:
\begin{align}
    \Set{N}_{i}=\mathop{\text{TopK}}\limits_{v_m\in\Set{C}_i}\left (\Mat{e}_i\cdot\Mat{e}_m\right ),
    \label{equ:eq_G}
\end{align}
where $\cdot$ is the inner product operation. 
$M$ is used to control the similarity between the center node and its immediate neighbor nodes, which can be demonstrated by the following proposition:
\begin{prop}\label{prop:graph}
Given an instance $v_i$ with the corresponding representation $\mathbf{e}_i$, assume that there are at least $S$ instances whose inner product similarity with $v_i$ is larger than $s$, i.e.,
\begin{equation}
 \left|\left\{v_j\in\mathcal{V}\mid \mathbf{e}_i\cdot \mathbf{e}_j>s\right\}\right|\geq S.   
\end{equation}
Then in the proximity graph $G$, the similarity between $v_i$ and its neighbors is larger than $s$ with proximate probability at least:
\begin{equation}
\label{equ:prob_neig_simi}
\mathbb{P}\left\{\mathbf{e}_i\cdot\mathbf{e}_k>s,\forall v_k\in\mathcal{N}_i\right\} \gtrapprox\left(1-p^M\right)^K,
\end{equation}
where $p=\frac{N-S}{N}$, and $K$ is the number of neighbors.
\end{prop}

\begin{proof}
Since $M\ll N$, we can approximately assume that the sampling is with replacement. In this case, we have 
\begin{equation}
\mathbb{P}\left\{\mathbf{e}_i\cdot\mathbf{e}_k>s,\forall v_k\in\mathcal{N}_i\right\}=1-\sum_{k=0}^{K-1}{M\choose k}p^{M-k}\left(1-p\right)^{k}.
\end{equation}
Then let us prove \eqref{equ:prob_neig_simi} by induction. When $K=1$, the conclusion clearly holds.

Assuming that the conclusion holds when $K=L-1$, let us consider the case when $K=L$. We have
\begin{small}
\begin{equation}
\begin{aligned}
1-\sum_{k=0}^{L-1}{M\choose k}p^{M-k}(1-p)^{k}\gtrapprox \left(1-p^M\right)^{L-1}-{M\choose L-1}p^{M-L+1}\left(1-p\right)^{L-1}.
\end{aligned}
\end{equation}
\end{small}
To prove the conclusion, we only need to show
\begin{equation}
\left(1-p^M\right)^{L-1}p^M\gtrapprox  {M\choose L-1}p^{M-L+1}\left(1-p\right)^{L-1},
\end{equation}
or equivalently
\begin{equation}
\label{equ:pf_first}
\begin{aligned}
\left(1-p^M\right)^{L-1}p^{L-1}&=\left(1-p^M\right)^{L-1}\left(\frac{N-S}{N}\right)^{L-1}\\
&\gtrapprox {M\choose L-1}\left(\frac{S}{N}\right)^{L-1}= {M\choose L-1}\left(1-p\right)^{L-1}.
\end{aligned}
\end{equation}
On the other hand, according to \cite{knuth1997art}, we have
\begin{equation}
\label{equ:pf_second}
 {M\choose L-1}  \leq \left(\frac{eM}{L-1}\right)^{L-1},
\end{equation}
where $e$ denotes the Euler's number.
Substituting \eqref{equ:pf_second} into \eqref{equ:pf_first}, we only need to show
\begin{equation}
(N-S)(L-1)\left(1-p^M\right) \gtrapprox eMS.
\end{equation}
The above relation holds depending on the choices of $M$, $S$ and $L$, which can be satisfied in our scenario in most cases.
\end{proof}

Proposition~\ref{prop:graph} suggests that the candidate set size $M$ can control the similarity between each center node and its immediate neighbor nodes. 
A larger $M$ indicates a greater probability that two adjacent nodes are similar, and the \graph constructed would be more like the graph of nodes clustered by kNN. 
If $M$ is small and close to $K$, the instances in the proximity graph can be considered as randomly-selected ones, that is, by Uniform Sampler.

\subsection{Proximity Graph Sampling} 

We perform the mini-batch sampling as a walk in the \graph, which collects the visited instances as sampling results from a sourced node. Here we propose to apply Random Walk with Restart~(RWR), which offers a theoretically supported ability to control the walker's behavior.

As shown in Algorithm~\ref{rwr}, starting from a node, the sampler moves from one node to another by either teleporting back to the start node with probability $\alpha$ or moving to a neighboring node proportional to the edge weight. The process continues until a fixed number of nodes are collected and taken as the mini-batch. The effectiveness of RWR lies in its ability to modulate the probability of sampling within a neighborhood by adjusting $\alpha$, as demonstrated by Proposition~\ref{prop:rwr}:
\begin{prop}\label{prop:rwr} 
For all $0<\alpha\le 1$ and $\Set{S}\subset\Set{V}$, the probability that a Lazy Random Walk with Restart starting from a node $u\in\Set{S}$ escapes $\Set{S}$ satisfies
$\sum_{v\in(\Set{V}-\Set{S})}\Mat{p}_u(v)\le\frac{1-\alpha}{2\alpha}\Phi(\Set{S})$, where $\Mat{p}_u$ is the stationary distribution, and $\Phi(\Set{S})$ is the graph conductance of $\Set{S}$.
\end{prop}

\begin{proof}
We first introduce the definition of graph conductance~\citep{sima2006} and Lazy Random Walk~\citep{lc2013}:

\noindent \textit{Graph Conductance.} For an undirected graph $G=(\Set{V},\Set{E})$, the graph volume of a node set $\Set{S}\subset\Set{V}$ is defined as $\text{vol}(\Set{S})=\sum_{v\in\Set{S}}d(v)$, where $d(v)$ is the degree of node $v$. The edge
boundary of a node set is defined to be
$\partial(\Set{S})=\{(x, y)\in\Set{E}|x\in\Set{S},y\notin\Set{S}\}$. The conductance of $\Set{S}$ is calculated as followed:
\begin{align}
\Phi(\Set{S})=\frac{|\partial(\Set{S})|}{\min(\text{vol}(\Set{S}), \text{vol}(\Set{V}-\Set{S}))}
\end{align}

\noindent \textit{Lazy Random Walk.} Lazy Random Walk~(LRW) is a variant of Random Walk, which first starts at a node, then stays at the current position with a probability of 1/2 or travels to a neighbor. The transition matrix of a lazy random walk is $\Mat{M}\triangleq(\Mat{I}+\Mat{A}\Mat{D}^{-1})/2$, where the $\Mat{I}$ denotes the identity matrix, $\Mat{A}$ is the adjacent matrix, and $\Mat{D}$ is the degree matrix. The $K$-th step Lazy Random Walk distribution starting from a node $u$ is defined as $\Mat{q}^{(K)}\leftarrow\Mat{M}^K\Mat{1}_u$. 

We then present a theorem that relates the Lazy Random Walk to graph conductance, which has been proved in \citet{lc2013}:

\begin{thm}\label{equ:theory1}
For all $K\ge 0$ and $\Set{S}\subset\Set{V}$, the probability that a $K$-step Lazy Random Walk starting at $u\in\Set{S}$ escapes $\Set{S}$ satisfies $\Mat{q}^{(K)}(\Set{V}-\Set{S})\le K\Phi(\Set{S})/2$.
\end{thm}

Theorem~\ref{equ:theory1} guarantees that given a non-empty node set $\Set{S}\subset\Set{V}$ and a start node $u\in\Set{S}$, the Lazy Random Walker will be more likely stuck at $\Set{S}$. 
Here we extend the LRW to Lazy Random Walk with Restart~(LRWR) which will return to the start node with probability $\alpha$ or perform Lazy Random Walk.
According to the previous studies~\citep{page1999pagerank,avrachenkov2014pagerank,chung2010pagerank,tong2006}, we can obtain a stationary distribution $\Mat{p}_u$ by recursively performing Lazy Random Walk with Restart, which can be formulated as a linear system:
\begin{equation}
\begin{aligned}
\Mat{p}_u=\alpha\Mat{1}_u+(1-\alpha)\Mat{M}\Mat{p}_u
\end{aligned}
\end{equation}
where $\alpha$ denotes the restart probability. 
$\mathbf{p}_u$ can be expressed as a geometric sum of Lazy Random Walk~\citep{chung2010finding}: 
\begin{equation}
\begin{aligned}
\Mat{p}_u=\alpha\sum_{l=0}^{\infty}(1-\alpha)^{l}\mathbf{M}^l\mathbf{1}_u=\alpha\sum_{l=0}^{\infty}(1-\alpha)^{l}\mathbf{q}_u^{(l)}
\end{aligned}
\end{equation}
Applying Theorem~\ref{equ:theory1}, we have: 
\begin{equation}
\begin{aligned}
\label{equ:p2_pf}
\sum_{v\in(\Set{V}-\Set{S})}\Mat{p}_u(v)
&=\alpha\sum_{l=0}^{\infty}\sum_{v\in(\Set{V}-\Set{S})}(1-\alpha)^{l}\Mat{q}_u^{(l)}(v) \\
&\le \alpha\sum_{l=0}^{\infty}l(1-\alpha)^l\Phi(\Set{S})/2 \\
&=\frac{1-\alpha}{2\alpha}\Phi(\Set{S})
\end{aligned}
\end{equation}
where the element $\Mat{p}_u(v)$ represents the probability of the walker starting at $u$ and ending at $v$. The desired result is obtained by comparing two sides of \eqref{equ:p2_pf}.
\end{proof}

The only difference between Lazy Random with Restart and Random Walk with Restart is that the former has a probability of remaining in the current position without taking action. They are equivalent when sampling a predetermined number of nodes. 

Overall, Proposition~\ref{prop:rwr} indicates that the probability of RWR escaping from a local cluster~\citep{lc2006,lc2013} can be bounded by the graph conductance~\citep{sima2006} and the restart probability $\alpha$. Besides, RWR can exhibit a mixture of two straightforward sampling methods Breadth-first Sampling (BFS) and Depth-first Sampling (DFS)~\citep{grover2016node2vec}:
\begin{itemize}[leftmargin=*]
\item BFS collects all of the current node's immediate neighbors, then moves to its neighbors and repeats the procedure until the number of collected instances reaches batch size.
\item DFS randomly explores the node branch as far as possible before the number of visited nodes reaches batch size.
\end{itemize}
Specifically, higher $\alpha$ indicates that the walker will approximate BFS behavior and sample within a small locality, while a lower $\alpha$ encourages the walker to explore nodes further away, like in DFS.

\subsection{Discussion on \model} 
As shown in Algorithm~\ref{pipeline}, 
\model serves as a mini-batch sampler and can be easily plugged into any in-batch contrastive learning methods.
Specifically, during the training process, \model first constructs the \graph, which will be updated after $t$ training steps, then selects a start node at random and samples a mini-batch on \graph by RWR.

\begin{algorithm}[t]
    \caption{Constrative Learning with \model}
    \label{pipeline}
    \KwIn{Dataset $\Set{D}=\{x_i|i=1,\cdots,N\}$,
    Encoder $f(\cdot)$, Batchsize $B$, Graph update interval $t$, Modality-specific augmentation functions $\mathcal{T}$.
    \\ }
    \For{iter $\leftarrow$ $0,1,\cdots$}{
        \tcbset{boxrule = 0.25mm, colback=red!5!white, colframe=red!75!black}
        \begin{tcolorbox}[colback=red!5!white,colframe=red!75!black, width=2.8in]
            \vspace{-0.2cm}
            \textcolor{comment}{// \model} \\
            \If{iter$\%t == 0$}{
                \textcolor{comment}{// Proximity Graph Construction} \\
              {Build the \graph $G$ by Algorithm~\ref{dknn}.}
             }
            {\textcolor{comment}{// Proximity Graph Sampling} \\}
            {Perform sampling in $G$ by Algorithm~\ref{rwr}.}
            \vspace{-0.2cm}
        \end{tcolorbox}
        \vspace{-0.1cm}
        \textcolor{comment}{// Standard Contrastive Pipeline in Different Modalities} \\
        {Load the mini-batch $\{x_i\}_B$.} \\
        {Obtain positive pairs $\{(x_i, x^+_i)\}_B$ by augmentations $\mathcal{T}$.} \\
        {Generate representations $\{(\Mat{e}_i, \Mat{e}^+_i)\}_B$ by Encoder $f(\cdot)$.} \\
        {Compute the loss by treating $\{(\Mat{e}_i, \Mat{e}_j)\}^{i\ne j}$ as negatives.} \\
        {Update the parameters of $f(\cdot)$.} \\
    }
\end{algorithm}

\vpara{Connects to the Uniform and kNN Samplers.} As shown in Figure~\ref{fig:per_sampler}, the number of candidates $M$ and the restart probability $\alpha$ are the key to flexibly controlling the hardness of a sampled batch. When we set $M$ as the size of dataset and $\alpha$ as 1, \graph is equivalent to kNN graph and graph sampler will only collect the immediate neighbors around a central node, which behaves similarly to a kNN Sampler.
On the other hand, if $M$ is set to 1 and $\alpha$ is set to 0, the RWR degenerates into the DFS and chooses the neighbors that are linked at random, which indicates that \model performs as a Uniform Sampler.
We provide an empirical criterion of choosing $M$ and $\alpha$ in Section~\ref{sec:criterion}.

\vpara{Complexity.} The time complexity of building a \graph is $O(NMd)$ where $N$ is the dataset size, $M$ is the candidate set size and $d$ denotes the embedding size. It is practically efficient since usually $M$ is much smaller than $N$, and the process can be accelerated by embedding retrieval libraries such as Faiss~\citep{faiss2019}. The space cost of \model mainly comes from  graph construction and graph storage. The total space complexity of \model is $O(Nd+NK)$ where $K$ is the number of neighbors in the \graph.

\section{Experiments} \label{sec:experiment}

We plug \model into typical contrastive learning algorithms on three modalities, including vision, language, and graphs. Extensive experiments are conducted with 5 algorithms and 19 datasets, a total of 31 experimental settings. 
Additional experiments are reported in Appendix~\ref{sec:appendix_exp}, including batchsize $B$, neighbor number $K$, \graph update interval $t$, and the training curves. InfoNCE objective and its variants are described in Appendix~\ref{sec:infonce_loss}.
The statistics of the datasets are reported in Appendix~\ref{sec:dataset_detail}, and the detailed experimental setting can be found in Appendix~\ref{sec:experimental_settings}.

\subsection{Results}\label{sec:exp_combine}

\vpara{Results on Vision.} 
We first adopt SimCLR~\citep{simclr2020} and MoCo v3~\citep{chen2021empirical} as the backbone based on ResNet-50~\citep{resnet2016}. We start with training the model for 800 epochs with a batch size of 2048 for SimCLR and 4096 for MoCo v3, respectively. We then use linear probing to evaluate the representations on ImageNet. 
As shown in Table~\ref{tab:imagenet_result}, our proposed model can consistently boost the performance of original SimCLR and MoCo v3, demonstrating the superiority of \model.
Besides, we evaluate \model on the other benchmark datasets: two small-scale (CIFAR10,
CIFAR100) and two medium-scale (STL10, ImageNet-100), which can be found in Appendix~\ref{sec:appendix_cv}.

\begin{table*}\footnotesize
\begin{minipage}[t]{0.3\textwidth}
    \centering
    \caption{Top-1 accuracy under the linear evaluation with the ResNet-50 backbone on ImageNet.}
    \label{tab:imagenet_result}
\setlength{\tabcolsep}{2mm}{
\renewcommand\arraystretch{1.15}
  \begin{threeparttable}
    \begin{tabular}{l|ccc} 
        \toprule
        Method   & 100 ep & 400 ep & 800 ep\\
        \midrule
         SimCLR     & 64.0  & 68.1 & 68.7 \\
         w/ \model   &  \textbf{64.7} &  \textbf{68.6} & \textbf{69.2} \\
         \hline
         \hline 
         MoCo v3   &  68.9 & 73.3 & 73.8 \\
         w/ \model  & \textbf{69.5}  & \textbf{73.7}  & \textbf{74.2}  \\
        \bottomrule
    \end{tabular}
    \begin{tablenotes}
            \footnotesize
            \item[*] Only conduct experiments on \model.  
          \end{tablenotes}
    \end{threeparttable}
    \vspace{-6.5pt}}
\end{minipage}\hfill
\begin{minipage}[t]{0.65\textwidth}
    \centering
    \caption{Overall performance comparison with the BERT backbone on STS tasks.}
    \label{tab:text}
\setlength{\tabcolsep}{2.5mm}{
\renewcommand\arraystretch{1.15}
\begin{tabular}{l|cccccccc}  
\toprule
 Method     & STS12 & STS13 &  STS14 & STS15 & STS16 & STS-B & SICK-R & Avg. \\
\midrule
  SimCSE-BERT$_{base}$ & 68.62  & 80.89    & 73.74 & 80.88    & 77.66 & \textbf{77.79} & \textbf{69.64} & 75.60    \\
  w/ kNN Sampler & 63.62 & 74.86 & 69.79 & 79.17 & 76.24 &    74.73     &     67.74      & 72.31  \\
  w/ \model         & \textbf{72.37} & \textbf{82.08} &\textbf{75.24}  &\textbf{83.10}    &\textbf{78.43} & 77.54 & 68.05  & \textbf{76.69}     \\
  \hline
  \hline
  DCL-BERT$_{base}$   & 65.22 & 77.89 & 68.94  & 79.88     & \textbf{76.72} & 73.89 & \textbf{69.54} & 73.15    \\      
  w/ kNN Sampler & 66.34 & 76.66 & 72.60 & 78.30 & 74.86 &    73.65     &     67.92      & 72.90  \\
  w/ \model     & \textbf{69.55} & \textbf{82.66} & \textbf{73.37}  & \textbf{80.40}   & 75.37 & \textbf{75.43} & 66.76 & \textbf{74.79}      \\
  \hline
  \hline 
  HCL-BERT$_{base}$   & 62.57  & 79.12 & 69.70   & 78.00  & 75.11    & 73.38 & 69.74 & 72.52  \\
  w/ kNN Sampler  & 61.12 & 75.73 & 68.43 & 76.64 & 74.78 &    71.22     &      68.04      & 70.85 \\
  w/ \model     &  \textbf{66.87} & \textbf{81.38} & \textbf{72.96} & \textbf{80.11} & \textbf{77.99} & \textbf{75.95} & \textbf{70.89} & \textbf{75.16}  \\
\bottomrule
\end{tabular}}
\end{minipage}
\end{table*}

\begin{small}
\begin{table*}[t!]
    \centering
    \caption{Accuracy on graph classification task under LIBSVM~\citep{libsvm2011} classifier.}
    \renewcommand{\arraystretch}{1.15}
    \begin{threeparttable}
    \setlength{\tabcolsep}{3mm}{
    \begin{tabular}{c|ccccccc}  
    \toprule
    Method & IMDB-B &  IMDB-M  & COLLAB & REDDIT-B  & PROTEINS & MUTAG & NCI1 \\
    \midrule
    GraphCL  &  70.90$\pm$0.53  & 48.48$\pm$0.38     & 70.62$\pm$0.23 & 90.54$\pm$0.25 &  74.39$\pm$0.45 & 86.80$\pm$1.34   & 77.87$\pm$0.41  \\
     w/ kNN Sampler & 70.72$\pm$0.35 & 47.97$\pm$0.97    & 70.59$\pm$0.14 & 90.21$\pm$.74 & 74.17$\pm$0.41 & 86.46$\pm$0.82    & 77.27$\pm$0.37  \\
     w/ \model        & \textbf{71.90$\pm$0.46} & \textbf{48.93$\pm$0.28} & \textbf{71.48$\pm$0.28} &    \textbf{90.88$\pm$0.16} & \textbf{75.04$\pm$0.67} & \textbf{87.78$\pm$0.93}  & \textbf{78.93$\pm$0.38}  \\ 
     \hline 
     \hline
    DCL   & 71.07$\pm$0.36  & 48.93$\pm$0.32 & \textbf{71.06$\pm$0.51}  & 90.66$\pm$0.29   & 74.64$\pm$0.48  & 88.09$\pm$0.93   & 78.49$\pm$0.48  \\
    w/ kNN Sampler & 70.94$\pm$0.19  & 48.47$\pm$0.35    & 70.49$\pm$0.37 & 90.26$\pm$1.03 & 74.28$\pm$0.17 & 87.13$\pm$1.40    & 78.13$\pm$0.52  \\
    w/ \model  & \textbf{71.32$\pm$0.17}  & \textbf{48.96$\pm$0.25}  &  70.44$\pm$0.35 &    \textbf{90.73$\pm$0.34} & \textbf{75.02$\pm$0.61} & \textbf{89.47$\pm$1.43}    & \textbf{79.03$\pm$0.32} \\
    \hline
    \hline
    HCL   & \textbf{71.24$\pm$0.36} & 48.54$\pm$0.51 & 71.03$\pm$0.45 & 90.40$\pm$0.42  & 74.69$\pm$0.42 & 87.79$\pm$1.10    & 78.83$\pm$0.67   \\
    w/ kNN Sampler & 71.14$\pm$0.44 & 48.36$\pm$0.93 & 70.86$\pm$0.74 & 90.64$\pm$0.51 & 74.06$\pm$0.44  & 87.53$\pm$1.37    & 78.66$\pm$0.48  \\
    w/ \model  & 71.20$\pm$0.38 &  \textbf{48.76$\pm$0.39} & \textbf{71.70$\pm$0.35}  &  \textbf{91.25$\pm$0.25}   & \textbf{75.11$\pm$0.63} & \textbf{88.31$\pm$1.29}    & \textbf{79.17$\pm$0.27}   \\
    \hline
    \hline
    MVGRL & 74.20$\pm$0.70 & 51.20$\pm$0.50 & -  & 84.50$\pm$0.60  & - & 89.70$\pm$1.10 & -  \\
    w/ kNN Sampler & 73.30$\pm$0.34 & 50.70$\pm$0.36  & - & 82.70$\pm$0.67 & - & 85.08$\pm$0.66 & - \\
    w/ \model & \textbf{76.70$\pm$0.35}  & \textbf{52.40$\pm$0.39} & - & \textbf{87.47$\pm$0.79} & - & \textbf{91.13$\pm$0.81} & -  \\
    \bottomrule
    \end{tabular}} 
    \begin{tablenotes}
        \footnotesize
        \item[*] The results not reported are due to the unavailable code or out-of-memory caused by the backbone model itself.
    \end{tablenotes}
    \end{threeparttable}
    \label{tab:graph}
\end{table*}
\end{small}

\vpara{Results on Language.} 
We evaluate \model on learning the sentence representations by SimCSE~\citep{simcse2021} framework with pretrained BERT~\citep{bert2018} as the backbone. The results of Table~\ref{tab:text} suggest that \model consistently improves the baseline models with an absolute gain of 1.09\%$\sim$2.91\% on 7 semantic textual similarity~(STS) tasks~\citep{sts2012,sts2013,sts2014,sts2015,sts2016,stsb2017,stsr2014}. Specifically, we observe that when applying DCL and HCL, the performance of the self-supervised language model averagely drops by 2.45\% and 3.08\% respectively. 
As shown in~\citet{zhou2022debiased} and Appendix~\ref{sec:score_comparison}, the pretrained language model offers a prior distribution over the sentences, 
leading to a high cosine similarity of both positive pairs and negative pairs.
So DCL and HCL, which leverage the similarity of positive and negative scores to tune the weight of negatives, are inapplicable because the
high similarity scores of positives and negatives will result in homogeneous weighting. However, the hard negatives explicitly sampled by our proposed \model can alleviate it, with an absolute improvement of 1.64\% on DCL and 2.64\% on HCL. The results of RoBERTa~\citep{liu2019roberta}  are reported in Appendix~\ref{sec:roberta}.

\vpara{Results on Graphs.} 
We test \model on graph-level classification task using GraphCL~\citep{graphcl2020} and MVGRL~\cite{hassani2020contrastive} frameworks.
Similar to language modality, we replace the original InfoNCE loss function in GraphCL with the DCL and HCL.
Table~\ref{tab:graph} reports the detailed performance comparison on 7 benchmark datasets: IMDB-B, IMDB-M, COLLAB, REDDIT-B, PROTEINS, MUTAG, and NCI1~\citep{yu2020graph,yu2021recognizing}.
We can observe that \model consistently boosts the performance of GraphCL and MVGRL, with an absolute improvement of $0.4\% \sim 2.9\%$ across all the datasets. 
Besides, equipped with \model, DCL and HCL can achieve better performance in 12 out of 14 cases.
It can also be found that \model can reduce variance in most cases, showing that the exploited hard negatives can enforce the model to learn more robust representations. 
We also compare \model with 11 graph classification models including the unsupervised graph learning methods~\cite{gin2018,ying2018hierarchical}, graph kernel methods~\cite{shervashidze2011weisfeiler,yanardag2015deep}, and self-supervised graph learning methods~\cite{gcc2020,graphcl2020,hassani2020contrastive,narayanan2017graph2vec,sun2019infograph,you2021graph,xu2021infogcl} (See Appendix~\ref{sec:graph_baselines}).

\subsection{Why does \model Perform Better?}\label{sec:sampler_compare}
In this section, we conduct experiments to investigate why \model performs better, using the example of vision modality. 
We evaluate SimCLR on CIFAR10 and CIFAR100, comparing the performance and false negatives of Uniform Sampler, kNN Sampler, and \model to gain a deeper understanding of \model. 
Figure~\ref{fig:per_sampler} presents the histogram of cosine similarity for all pairs within a sampled batch, and the corresponding  percentage of false negatives during training. 

The results indicate that although kNN Sampler can explicitly draw a data batch with similar pairs, it also leads to a substantially higher number of false negatives, resulting in a notable degradation of performance.
Uniform Sampler is independent of the model so the percentage of false negatives within the sampled batch remains consistent during training, but it fails to effectively sample the hard negatives.
\model can modulate $M$ and $\alpha$ to control the hardness of the sampled batch, which brings about the best balance between these two sampling methods, resulting in a mini-batch of hard-to-distinguish instances with fewer false negatives compared to the kNN Sampler.
Specifically, \model can sample hard mini-batch but only exhibits a slightly higher percentage of false negatives than Uniform Sampler with optimal parameter setting,
which enables \model to achieve the best performance. 
A similar phenomenon on CIFAR100 can be found in Appendix~\ref{sec:cifar10_sim_fn}.

\begin{table*}[t!]
    \centering
    \caption{Empirical criterion for \model on neighbor candidate size $M$ and restart probability $\alpha$.}
    \renewcommand{\arraystretch}{1.2}
    \setlength{\tabcolsep}{.9mm}{
    \begin{tabular}{c|c|ccccc|ccccc}
    \toprule
    \multirow{2}{*}{Modality} & \multirow{2}{*}{Dataset} & \multicolumn{5}{c|}{Neighbor \ Candidate \  Size \ $M$} & \multicolumn{5}{c}{Restart \ Probability \ $\alpha$} \\
    \cline{3-12}
    &  & 500 & 1000 & 2000 & 4000 & 6000 & 0.1 & 0.3 & 0.5 & 0.7 & 0.2$\sim$0.05 \\
    \midrule
    \multirow{4}{*}{Image}  & CIFAR10 & \textbf{92.54} & 92.49 & 91.83  & 91.72 & 91.43 & 92.41 & 92.26 & 92.12  & 92.06 & \textbf{92.54}  \\
    & CIFAR100 & 67.92 & \textbf{68.68} & 67.05 & 66.19 & 65.55 & 68.31 & 67.98 & 68.20  & 68.00 & \textbf{68.68} \\
    & STL10  & 84.16 & \textbf{84.38} & 82.80  & 81.91 & 80.92  & 83.01 & 80.69 & 83.93 & 82.56 & \textbf{84.38} \\
    & ImageNet-100 & 59.6 & \textbf{60.8} & 60.1  & 59.1 & 58.4 & 60.8 & 59.6 & 58.1 & 57.7 & \textbf{60.8} \\
    \hline
    Text & Wikipedia & 71.36 & \textbf{76.69} & 76.09   & 75.76 & 75.11  & 71.74 & 72.13 & 72.41 & \textbf{76.69} & -- \\
    \hline
    \multirow{4}{*}{Graph} & IMDB-B & \textbf{71.90$\pm$.46}  & 71.28$\pm$.51 & 71.13$\pm$.48 &  70.86$\pm$.56  & 70.68$\pm$.59 & 71.26$\pm$.29  & 71.00$\pm$.46 & 71.06$\pm$.21 & 70.78$\pm$.58   & \textbf{71.90$\pm$.46} \\
   & IMDB-M & \textbf{48.93$\pm$.28} & 48.68$\pm$.35  & 48.88$\pm$.94  & 48.71$\pm$.93 & 48.12$\pm$.75    & 48.48$\pm$1.07 & 48.27$\pm$.67 & 48.72$\pm$.41 & 48.78$\pm$.60  & \textbf{48.93$\pm$.28} \\ 
   & COLLAB  & 70.47$\pm$.33 & \textbf{71.48$\pm$.28} & 70.93$\pm$.50 & 70.46$\pm$.28 & 70.24$\pm$.56  & 70.36$\pm$.28 &  70.63$\pm$.53  & 70.63$\pm$.54 & 70.31$\pm$.37 & \textbf{71.48$\pm$.28}  \\
   & REDDIT-B & \textbf{90.88$\pm$.16} &  89.45$\pm$.99  & 90.64$\pm$.38 & 89.92$\pm$.75 & 90.37$\pm$.89 & 90.22$\pm$.38 & 89.51$\pm$.61 & 90.44$\pm$.48 & 90.28$\pm$.89  &\textbf{90.88$\pm$.16}  \\
    \bottomrule
    \end{tabular}} 
    \label{tab:criterion}
\end{table*}

\begin{figure}[h]
    \centering
    \includegraphics[width=0.48\textwidth]{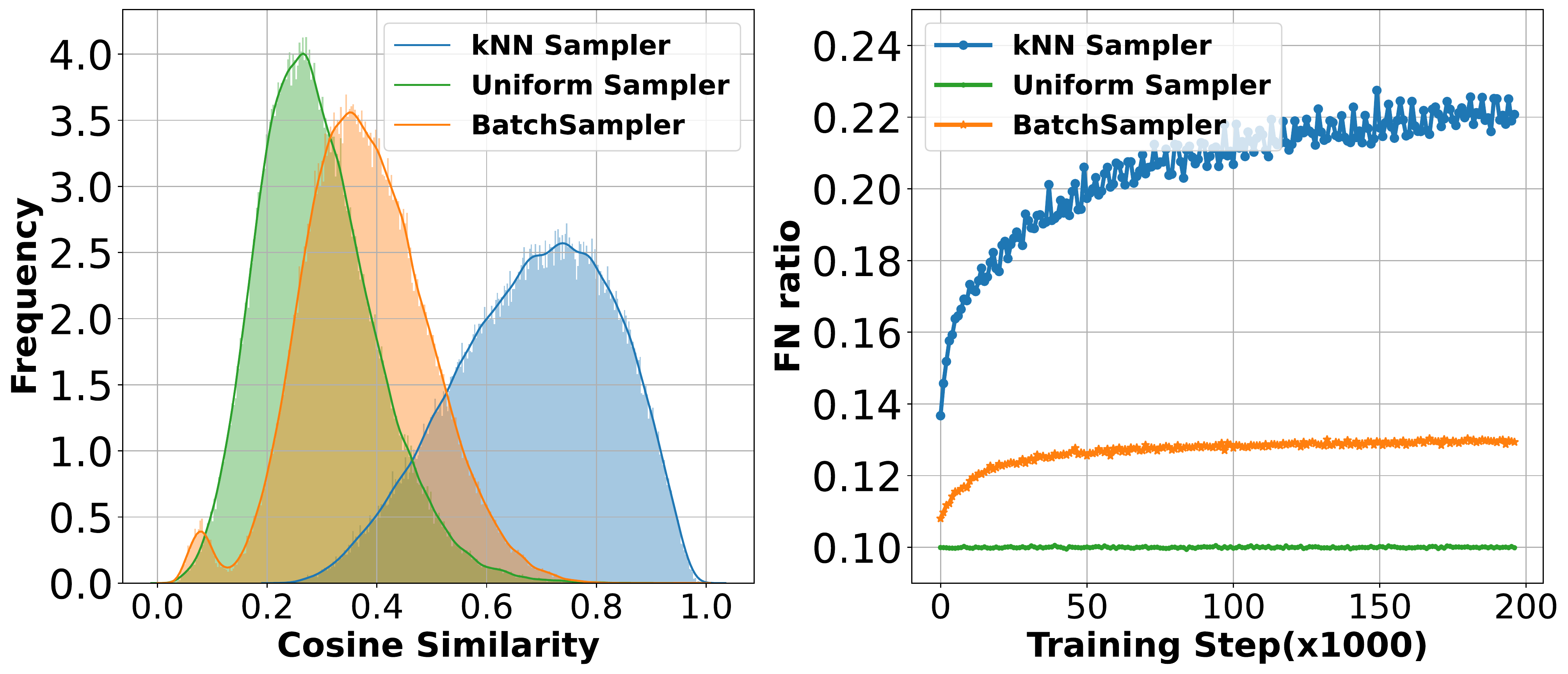}
    \vspace{-3mm}
     \caption{Cosine similarity and percentage of false negatives among various mini-batch sampling methods on CIFAR10.}
    \label{fig:per_sampler}
    \vspace{-2mm}
\end{figure}

\subsection{Empirical Criterion for \model}\label{sec:criterion}
\model modulates the hardness of the sampled batch by two important parameters $M$ and $\alpha$ to achieve better performance on three modalities.  
To analyze the impact of these, we vary the $M$ and $\alpha$ in the range of $\{500,1000,2000,4000,6000\}$ and $\{0.1,0.3,0.5,0.7\}$ respectively, and apply SimCLR, SimCSE and GraphCL as backones. 
We summarize the performance of \model with different $M$ and $\alpha$ in Table~\ref{tab:criterion}. It shows that in most cases, the performance of the model peaks when $M=1000$ but plumbs quickly with the increase of $M$. 
Such phenomena are consistent with the intuition that higher $M$ raises the probability of selecting similar instances as neighbors, but the sampler will be more likely to draw the mini-batch with false negatives, degrading the performance.

Besides, to better understand the effect of $\alpha$, we visualize the histograms of cosine similarity for all pairs from a sampled batch after training and plot the corresponding  percentage of false negatives during training on CIFAR10 (See Figure~\ref{fig:cos_sim}). 
We can observe that $\alpha$ moving from 0.1 to 0.7 causes cosine similarities to gradually skew left, but introduces more false negatives in the batch, creating a trade-off. This phenomenon indicates that the sampler with a higher $\alpha$ sample more frequently within a local neighborhood, which is more likely to yield similar pairs. 
A similar phenomenon can also be found on CIFAR100 (See Appendix~\ref{sec:restart_cifar100}). 
However, as training progresses, the instances of the same class tend to group together, increasing the probability of collecting false negatives. 
To find the best balance, we linearly decay $\alpha$ from 0.2 to 0.05 as the training epoch increases, which is presented as $0.2\sim 0.05$ in Table~\ref{tab:criterion}. It can be found that this dynamic strategy achieves the best performance in all cases except SimCSE which only trains for one epoch.
Interestingly, SimCSE achieves the best performance by a large margin when $\alpha=0.7$ since hard negatives can alleviate the distribution issue brought by the pre-trained language model. 
More analysis can be found in Section~\ref{sec:exp_combine} and Appendix~\ref{sec:score_comparison}.

\vpara{Criterion.}
From the above analysis, the suggested $M$ would be 500 for the small-scale dataset, and 1000 for the larger dataset.
The suggested $\alpha$ should be relatively high, e.g., 0.7, for the pre-trained language model-based method. 
Besides, dynamic decay $\alpha$, e.g., 0.2 to 0.05, is the best strategy for the other methods.
Such an empirical criterion provides critical suggestions for selecting the appropriate $M$ and $\alpha$ to modulate the hardness of the sampled batch.

\begin{figure}[h]
    \centering
    \includegraphics[width=0.48\textwidth]{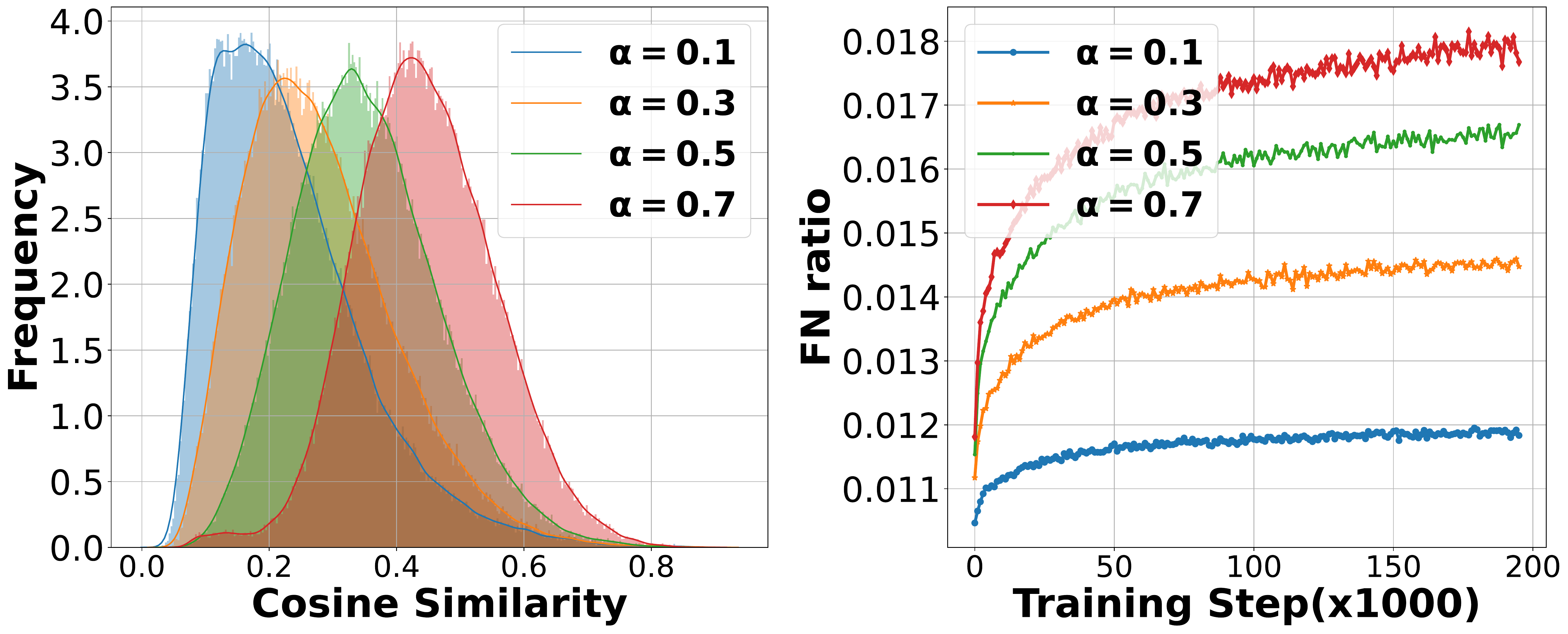}
    \vspace{-3mm}
    \caption{Cosine similarity and percentage of false negatives among various restart probabilities on CIFAR10.}
    \label{fig:cos_sim}
    \vspace{-5mm}
\end{figure}

\subsection{Efficiency Analysis}\label{sec:appendix_time}
Here, we introduce three metrics to analyze the time cost of mini-batch sampling by \model: (1) Batch Sampling Cost~(\textit{$\text{Cost}_{S}$}) is the average time of RWR taken to sample a mini-batch from a \graph; (2) Proximity Graph Construction Cost~(\textit{$\text{Cost}_{G}$}) refers to the time consumption of \model for constructing a \graph; 
(3) Batch Training Cost~(\textit{$\text{Cost}_{T}$}) is the average time taken by the encoder to forward and backward; 
(4) Proximity Graph Construction Amortized Cost~(\textit{$\text{Cost}_{G/t}$}) is the ratio of \textit{$\text{Cost}_{G}$} to the graph update interval $t$.
The time cost of \model is shown in Table~\ref{tab:timecost_iter}, from which we make the following observations:~(1) Sampling a mini-batch \textit{$\text{Cost}_S$} takes an order of magnitude less time than training with a batch \textit{$\text{Cost}_T$} at most cases.
(2) Although it takes $100 s$ for \model to construct a \graph in ImageNet, the cost shares across $t$ training steps, which take only \textit{$\text{Cost}_{G/t}=0.2$} per batch. A similar phenomenon can be found in other datasets as well. 
In particular, SimCSE trains for one epoch, and \graph is built once.

\begin{small}
    \begin{table}[h]
    \caption{The time cost of mini-batch sampling by \model on an NVIDIA V100 GPU.}
    \centering
    \renewcommand{\arraystretch}{1.15}
    \setlength{\tabcolsep}{1.mm}{
    \begin{tabular}{c|cccc}  
    \toprule
        Metric & STL10 & ImageNet-100 & Wikipedia & ImageNet  \\
        \midrule
         $\text{Cost}_{S}$   & 0.013s  & 0.015s & 0.005  & 0.15s \\ 
         {$\text{Cost}_{G}$}  & 2s  & 3s & 79s & 100s \\ 
         $\text{Cost}_{T}$   & 0.55s  & 1.1s  &  0.08s  &  1.1s \\ 
         \hline
         {$\text{Cost}_{G/t}$} & 0.02($t=100$) & 0.03($t=100$)  & 0.005($t=15625$) & 0.2($t=500$) \\ 
   \bottomrule
    \end{tabular}
    } 
     \label{tab:timecost_iter}
    \end{table}
\end{small}

\subsection{Further Analysis}

\vpara{Strategies of Proximity Graph Sampling.}
We conduct experiments to explore different choices of graph sampling methods, including (1)~Depth First Search~(DFS); (2)~Breadth First Search~(BFS); (3)~Random Walk~(RW); (4)~Random Walk with Restart~(RWR). 
Table~\ref{tab:sampling_result} presents an overall performance comparison with different graph sampling methods. 
As expected, RWR consistently achieves better performance since it samples hard yet true negatives within a local cluster on \graph.
Besides, we illustrate the histograms of cosine similarity for all pairs from a sampled batch and plot the corresponding percentage of false negatives during training in Figure~\ref{fig:graph_sampling}.
It can be observed that although BFS brings the most similar pairs in the mini-batch, it performs worse than the original SimCLR since it introduces substantial false negatives. 
While having a slightly lower percentage of false negatives than RWR, DFS, and RW do not exhibit higher performance since they are unable to collect the hard negatives in the mini-batch.
The restart property allows RWR to exhibit a mixture of DFS and BFS, which can flexibly modulate the hardness of the sampled batch and find the best balance between hard negatives and false negatives.

\begin{small}
\begin{table}[t!]
    \centering
    \caption{Overall performance comparison with different graph sampling methods.}
    \renewcommand{\arraystretch}{1.15}
    \setlength{\tabcolsep}{1.0mm}{
    \begin{tabular}{c|ccc|c|cc}
    \toprule
    \multirow{2}{*}{Method} & \multicolumn{3}{c|}{Vision} & Language & \multicolumn{2}{c}{Graphs} \\
     & CIFAR10 & CIFAR100 & STL10 & Wikipedia & IMDB-B  & COLLAB \\
    \midrule
    BFS & 91.03 & 65.15 & 77.08 & 74.39 & 70.48 & 69.98 \\
    DFS & 92.14 & 68.29 & 83.05 & 73.40  & 71.12 & 70.60 \\
    RW  & 92.28 & 68.33 & 83.54 & 75.56 & 71.26 &  70.72 \\
    RWR & \textbf{92.54} & \textbf{68.68} & \textbf{84.38} & \textbf{76.69} & \textbf{71.90} & \textbf{71.48} \\
    \bottomrule
    \end{tabular}} 
    \label{tab:sampling_result}
\end{table}  
\end{small}

\begin{figure}[h]
    \centering
    \includegraphics[width=0.48\textwidth]{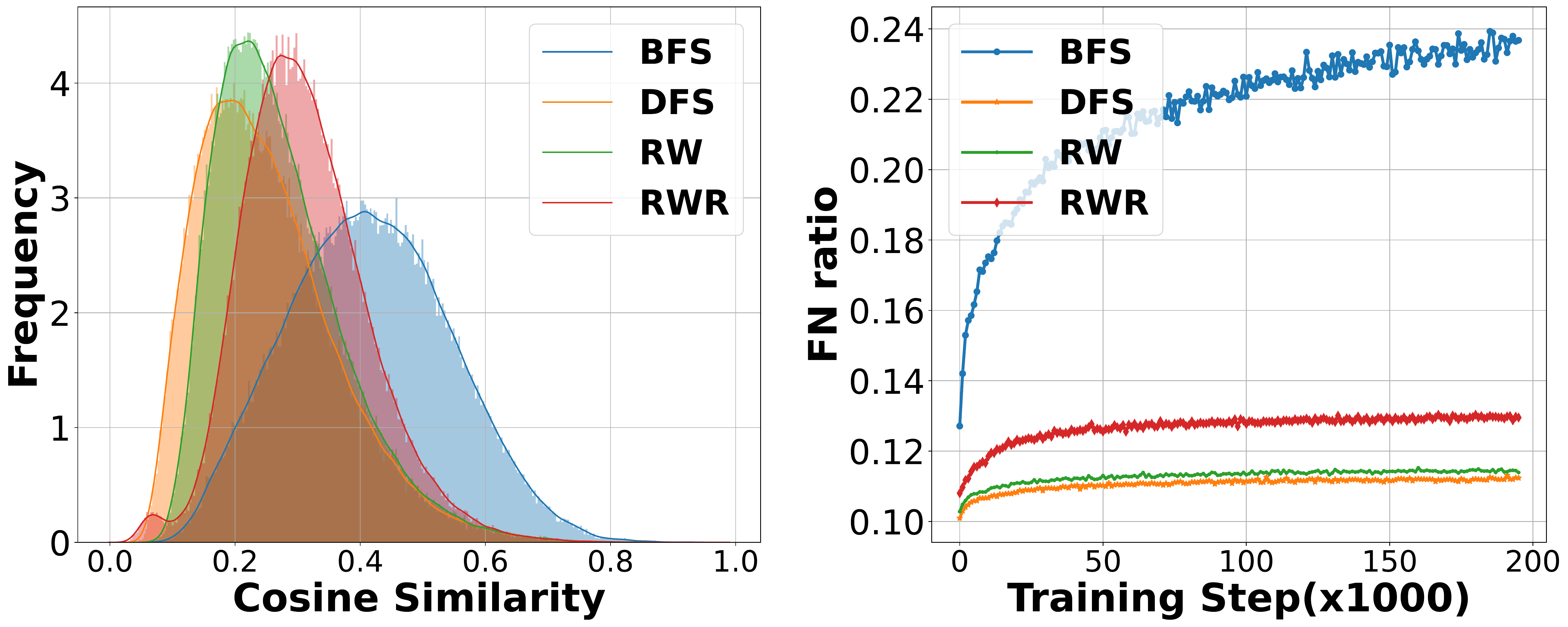}
    \caption{Histograms of cosine similarity and percentage of false negatives in a batch using different sampling methods.}    
    \label{fig:graph_sampling}
\end{figure}

\vpara{Proximity Graph vs. kNN Graph.}
To demonstrate the effectiveness of \graph, we do an ablation study by replacing \graph with kNN graph which directly selects $k$ neighbors with the highest scores for each instance from the whole dataset. 
The neighbor number $k$ is 100 by default.
The comparison results are shown in Table~\ref{tab:knngraph}, from which we can observe that \graph outperforms the kNN graph by a margin. \model with kNN graph even performs worse than the original contrastive learning method because of the false negatives.

\begin{small}
\begin{table}[hbpt]
    \centering
    \caption{Performance comparison of different graph construction methods.}
    \renewcommand{\arraystretch}{1.15}
    \setlength{\tabcolsep}{.8mm}{
    \begin{tabular}{c|cc|c|cc}
    \toprule
    \multirow{2}{*}{Method} & \multicolumn{2}{c|}{Vision} & Language & \multicolumn{2}{c}{Graphs}  \\
     & CIFAR10 & CIFAR100 & Wikipedia & IMDB-B  & COLLAB \\
    \midrule
    Default & 92.13  & 68.14 & 75.60 & 70.90 & 70.62 \\
    kNN graph & 90.47 & 62.67 & 75.13 & 70.10 & 69.96  \\
    \graph  & \textbf{92.54} & \textbf{68.68} & \textbf{76.69} & \textbf{71.90} & \textbf{71.48} \\
    \bottomrule
    \end{tabular}} 
    \label{tab:knngraph}
\end{table}  
\end{small}

To develop an intuitive understanding of how proximity graph alleviates the false negative issue, Figure~\ref{fig:knngraph} plots the changing curve of false negative ratio in a batch. 
The results show that the proximity graph 
could discard the false negative significantly: 
by the end of the training, kNN will introduce more than 22\% false negatives in a batch, while the proximity graph brings about 13\% on CIFAR10. 
A similar phenomenon can also be found on CIFAR100.

\begin{figure}[hbpt]
    \centering
    \includegraphics[width=0.48\textwidth]{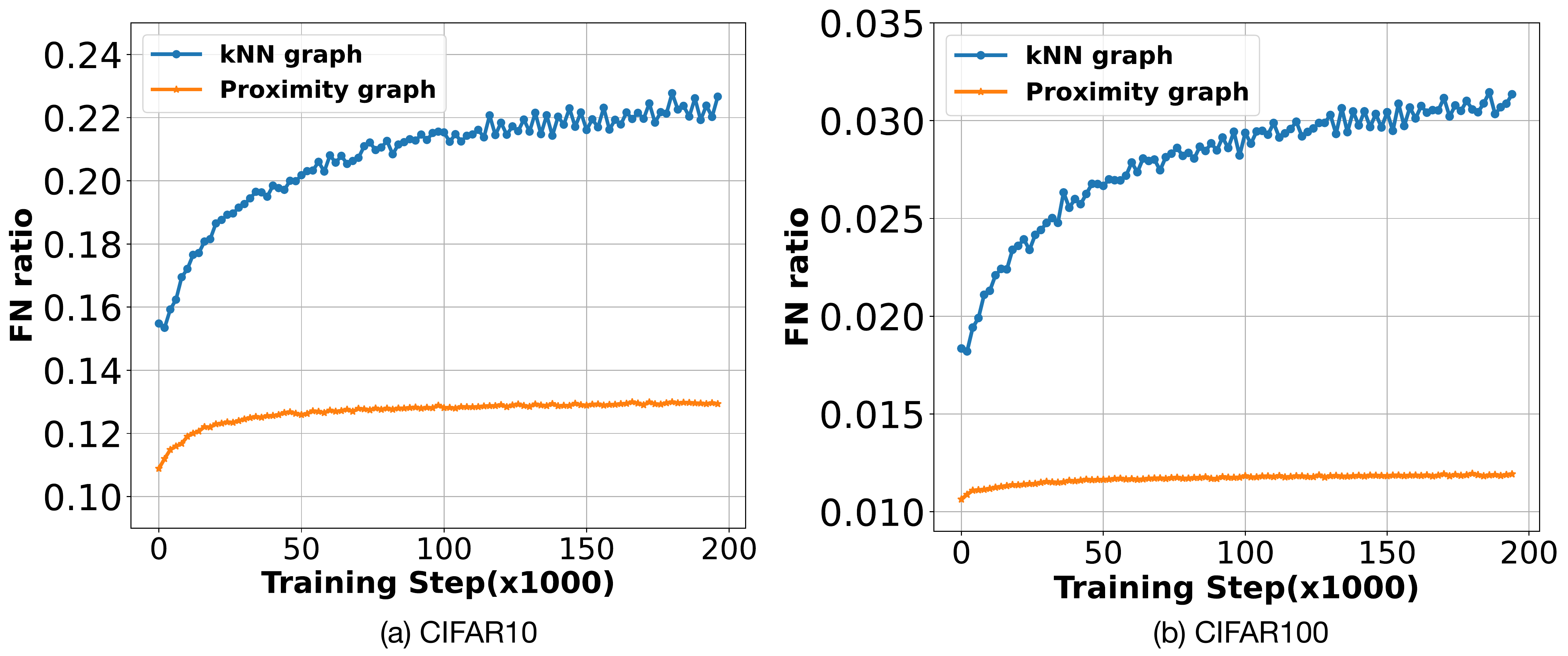}
    \vspace{-10pt}
    \caption{Percentage of false negatives using different graph building methods over the training step.}
    \label{fig:knngraph}
    \vspace{-7pt}
\end{figure}

\section{Conclusion}  \label{sec:conclusion}
In this paper, 
we study the problem of mini-batch sampling for in-batch contrastive learning, which aims to globally sample mini-batches of hard-to-distinguish instances. To achieve this, we propose \model to perform mini-batch sampling as a walk on the constructed \graph. 
Specifically, we design the \graph to control the pairwise similarity among instances and leverage random walk with restart (RWR) on the \graph to form a batch.
We theoretically and experimentally demonstrate that \model can balance kNN Sampler and Uniform Sampler. Besides, we conduct experiments with 5 representative contrastive learning algorithms on 3 modalities (e.g. vision, language, and graphs) to evaluate our proposed \model, demonstrating that \model can consistently achieve performance improvements.

\clearpage

\vpara{Acknowledgements.} 
This work was supported by Technology and Innovation Major Project of the Ministry of Science and Technology of China under Grant 2020AAA0108400 and 2020AAA0108402, NSF of China for Distinguished Young Scholars (61825602), NSF of China (62276148), and a research fund from Zhipu.AI.

\balance

\bibliographystyle{ACM-Reference-Format}
\bibliography{mybib}

\appendix
\clearpage

\section{Algorithm Details}\label{sec:algorithm}

Here we present the details of Proximity Graph Construction and Random Walk with Restart, as shown in Algorithm~\ref{dknn} and Algorithm~\ref{rwr} respectively.

\setlength\intextsep{12pt}
\begin{algorithm}[h]
    \caption{Proximity Graph Construction}
    \label{dknn}
    \KwIn{Dataset $\Set{D}=\{x_i\}$, Candidate set size $M$, Neighbor number $K$; \\ }
    \KwOut{A \graph $G$;}
    \For{$v$ in $\Set{D}$}{
        {Randomly select $M$ neighbor candidates from $\Set{D}$}; \\
        {Select the $K$ closest candidates $\Set{N}_v$ by Eq.~\ref{equ:eq_G}}; \\
        {$G[v]\leftarrow\Set{N}_v$};
    }
    \Return $G$
\end{algorithm}

\begin{algorithm}[h]
    \caption{Random Walk with Restart(RWR)}
    \label{rwr}
    \KwIn{Proximity graph $G=\{\Set{V},\Set{E}\}$, seed node $u$, restart probability $\alpha$, number of sampled node $B$; \\ }
    \KwOut{A sampled node set $\Set{S}$;}
    $\Set{S} \leftarrow \{\}, v \leftarrow u$; \\
    \While{$\text{len}(\Set{S}) < B$}{
	    \If{$v$ not in $\Set{S}$}{
	        $\Set{S}$.insert($v$)
	    }
	    Sample $r$ from Uniform distribution $U(0,1)$; \\
	    \If{$r< \alpha$}{
	        $v \leftarrow u$;
	    }
	    \Else{
	        {Randomly sample $\hat{v}$ from $v$'s neighbors}; \\
	        {$v \leftarrow \hat{v}$};
	    }
	}
    \Return{$\Set{S}$}
\end{algorithm}

\section{Additional Experiments}\label{sec:appendix_exp}

\subsection{Extensive studies on Vision}\label{sec:appendix_cv}

Here we evaluate the \model on two small-scale~(CIFAR10, CIFAR100) and two medium-scale~(STL10, ImageNet-100) benchmark datasets, and equip DCL~\citep{debias2021} and HCL~\citep{hard2021}\footnote{DCL and HCL are more like variants of InfoNCE loss, which adjust the weights of negative samples in the original InfoNCE loss.}
with \model to investigate its generality.
Experimental results in Table~\ref{tab:image} show that \model can consistently improve SimCLR and its variants on all the datasets, with an absolute gain of 0.3\%$\sim$2.5\%. 
We also can observe that the improvement is greater on medium-scale datasets than on small-scale datasets.
Specifically, the model equipped with HCL and \model achieves a significant improvement~(6.23\%) on STL10 over the original SimCLR.

\subsection{Experiments with Roberta on Language} \label{sec:roberta}
We also apply \model to the SimCSE with the pretrained RoBERTa~\cite{liu2019roberta} and present the results in Table~\ref{tab:text_roberta}. Similar to the results of BERT, \model can consistently improve the performance of the baseline model. 
Besides, as discussed in Section~\ref{sec:exp_combine}
and Section~\ref{sec:score_comparison}, 
the hard negative sampled by \model explicitly can alleviate the low distribution gap between positive score and negative score distribution caused by the pretrained language model, alleviating the performance degradation of DCL and HCL.

\begin{small}
    \begin{table}[h]
    \caption{Overall performance comparison on image classification task in terms of Top-1 Accuracy.}
    \centering
    \renewcommand{\arraystretch}{1.15}
    \setlength{\tabcolsep}{2mm}{
    \begin{tabular}{l|cccc}  
    \toprule
     Method     & CIFAR10  &  CIFAR100  & STL10 & ImageNet-100 \\
    \midrule
     SimCLR   & 92.13   &  68.14 & 83.26 & 59.30 \\
     w/ kNN Sampler & 90.16  & 62.30  & 79.25  & 57.70 \\
      w/ \model        & \textbf{92.54}  &\textbf{68.68} &  \textbf{84.38} & \textbf{60.80} \\
     \hline      
     \hline
     DCL   & 92.28  &  68.52 & 84.92 & 59.90 \\
     w/ \model  & \textbf{92.74}  & \textbf{68.91} & \textbf{86.39}  & \textbf{60.14} \\
     \hline 
     \hline
    HCL   & 92.39   & 68.92  & 88.20  & 60.60 \\ 
    w/ \model & \textbf{92.41}  & \textbf{69.13} & \textbf{89.49} & \textbf{61.50} \\
    \bottomrule
    \end{tabular}} 
    \label{tab:image}
    \end{table}
\end{small}

\begin{small}
    \begin{table}[h]
    \caption{Performance comparison for sentence embedding learning based on RoBERTa.}
    \centering
    \renewcommand{\arraystretch}{1.15}
    \setlength{\tabcolsep}{0.25mm}{
    \begin{tabular}{l|cccccccc}  
    \toprule
     Method     & STS12 & STS13 &  STS14 & STS15 & STS16 & STS-B & SICK-R & Avg. \\
    \midrule
      SimCSE-RoBERTa$_{base}$  & 67.90 & 80.91 & 73.14 & 80.58 & 80.74 & 80.26 & 69.87 & 76.20 \\ 
      w/ kNN Sampler & 68.78 & 79.49 & 73.34 & 81.05 & 80.15 &    77.09     &      67.18      & 75.30 \\
      w/ \model         & \textbf{68.29} & \textbf{81.96} &\textbf{73.86}  &\textbf{82.16}    &\textbf{80.94} & \textbf{80.77} & \textbf{69.30} & \textbf{76.75}      \\
     \hline 
      DCL-RoBERTa$_{base}$   & \textbf{66.60} & 79.16 & \textbf{71.05}  &   80.40  & 77.76 & \textbf{77.94} & \textbf{67.57} & 74.35   \\
      w/ kNN Sampler & 65.39 & 79.04 & 69.71 & 78.37 & 75.98 &   74.72    &      64.39      & 72.51 \\
      w/ \model  & 65.53 & \textbf{80.09} & 71.00  &\textbf{80.64}    & \textbf{78.35}   & 77.75 & 67.52 & \textbf{74.41}   \\
      \hline
      HCL-RoBETa$_{base}$   & \textbf{67.20} & 80.47 & 72.44  &  80.88  &  80.57 & 78.79 & 67.98 & 75.49    \\
      w/ kNN Sampler & 65.99 & 77.32 & 73.71 & 80.59 & 79.78 &  77.70  &      65.40    & 74.36 \\
      w/ \model & 66.01 & \textbf{80.79} & \textbf{73.58}  & \textbf{81.25}    & \textbf{80.66} & \textbf{79.22} & \textbf{68.52} & \textbf{75.72}      \\
    \bottomrule
    \end{tabular}} 
    \label{tab:text_roberta}
    \end{table}
\end{small}

\begin{table*}[t!]
    \centering
    \caption{Experiment results for graph classification task under LIBSVM~\citep{libsvm2011} classifier.}
    \renewcommand{\arraystretch}{1.2}
    \setlength{\tabcolsep}{1.mm}{
    \begin{tabular}{c|c|ccccccc}  
    \toprule
    Method & Dataset & IMDB-B &  IMDB-M  & COLLAB & REDDIT-B  & PROTEINS & MUTAG & NCI1 \\
    \midrule
    \multirow{2}{*}{Supervised} & GIN & 75.1$\pm$5.1 & 52.3$\pm$2.8 & 80.2$\pm$1.9 & 92.4$\pm$2.5 &76.2$\pm$2.8 & 89.4$\pm$5.6 & 82.7$\pm$1.7 \\
    & DiffPool & 72.6$\pm$3.9 & -  &78.9$\pm$2.3 & 92.1$\pm$2.6 & 75.1$\pm$3.5 & 85.0$\pm$10.3   & - \\
     \hline 
     \hline
   \multirow{2}{*}{Graph Kernels} & WL & 72.30$\pm$3.44 & 46.95$\pm$0.46 &  - & 68.82$\pm$0.41 & 72.92$\pm$0.56  & 80.72$\pm$3.00  & 80.31$\pm$0.46 \\
   & DGK & 66.96$\pm$0.56 & 44.55$\pm$0.52 & - & 78.04$\pm$0.39  & 73.30$\pm$0.82 & 87.44$\pm$2.72 & 80.31$\pm$0.46 \\
    \hline
    \hline
    \multirow{8}{*}{Self-supervised} & graph2vec & 71.10$\pm$0.54 & 50.44$\pm$0.87 & - & 75.78$\pm$1.03  & 73.30$\pm$2.05 & 83.15$\pm$9.25  & 73.22$\pm$1.81 \\
    & infograph & 73.03$\pm$0.87 & 49.69$\pm$0.53  & 70.65$\pm$1.13 & 82.50$\pm$1.42 & 74.44$\pm$0.31 & 89.01$\pm$1.13  & 76.20$\pm$1.06 \\
    & JOAO & 70.21$\pm$3.08 & 49.20$\pm$0.77 & 69.50$\pm$0.36 & 85.29$\pm$1.35  & \underline{74.55$\pm$0.41}  & 87.35$\pm$1.02  & 78.07$\pm$0.47 \\
    & GCC & 72.0 & 49.4 & \textbf{78.9} & \underline{89.9} & - & - & -\\
    & InfoGCL & \underline{75.10$\pm$0.90} & \underline{51.40$\pm$0.80} & - & - & - & \textbf{91.20$\pm$1.30} & - \\
    \cline{2-9}
    \cline{2-9}
     & GraphCL & 70.90$\pm$0.53 & 48.48$\pm$0.38 & 70.62$\pm$0.23 & 90.54$\pm$0.25 &74.39$\pm$0.45 &  86.80$\pm$1.34 & 77.87$\pm$0.41 \\
    & GraphCL + \model & 71.90$\pm$0.46 & 48.93$\pm$0.28 & \underline{71.48$\pm$0.28} & \textbf{90.88$\pm$0.16}  & \textbf{75.04$\pm$0.67} &  87.78$\pm$0.93 & \underline{78.93$\pm$0.38} \\
    \cline{2-9}
    \cline{2-9}
    & MVGRL & 74.20$\pm$0.70 & 51.20$\pm$0.50 & - & 84.50$\pm$0.60 & - & 89.70$\pm$1.10 & - \\
    & MVGRL + \model & \textbf{76.70$\pm$0.35} & \textbf{52.40$\pm$0.39} & - & 87.47$\pm$0.79 & - & \underline{91.13$\pm$0.81} & -  \\   
    \bottomrule
    \end{tabular}}
    \label{tab:graph_baseline}
\end{table*}

\subsection{Comparison with baselines on Graphs}\label{sec:graph_baselines}

In Table~\ref{tab:graph_baseline}, we comprehensively compare different kinds of baselines on graph classification tasks, including the unsupervised graph learning methods~\cite{gin2018,ying2018hierarchical}, graph kernel methods~\cite{shervashidze2011weisfeiler,yanardag2015deep}, and self-supervised graph learning methods~\cite{gcc2020,graphcl2020,hassani2020contrastive,narayanan2017graph2vec,sun2019infograph,you2021graph,xu2021infogcl}. 
It can be found that InfoGCL achieves state-the-of-art performance among the self-supervised methods, and even outperforms the supervised methods on some datasets. 
\model can consistently improve the performance of both GraphCL and MVGRL on all the datasets, demonstrating the effectiveness of global hard negatives. Benefiting from the performance gain brought by \model, MVGRL achieves the best performance on 4 datasets, outperforming InfoGCL and supervised methods.

\subsection{Performance Analysis on CIFAR100}\label{sec:cifar10_sim_fn}

Here we compare the Uniform Sampler, kNN Sampler, and \model regarding cosine similarity, false negatives, and performance comparison on CIFAR100. Specifically, we show the histogram of cosine similarity for all pairs in a sampled batch, the false negative ratio of the mini-batch, and performance comparison between \model and other two sampling strategies (e.g. Uniform Sampler and kNN Sampler) in Figure~\ref{fig:sampler_sim_fn}. It can be found that 
\model exhibits a balance of Uniform Sampler and kNN Sampler, which can sample the hard negative pair but only brings a slightly greater number of false negatives than Uniform Sampler. 
More analysis can be found in Section~\ref{sec:sampler_compare}.

\begin{figure*}[t!]
    \centering
    \includegraphics[width=0.8\textwidth]{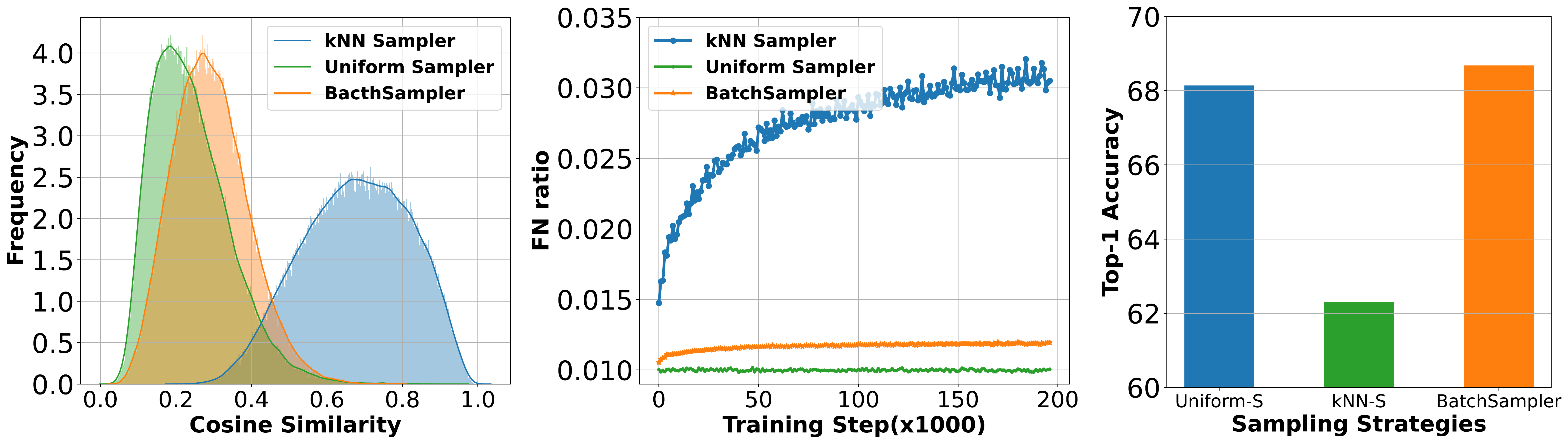}
	\caption{Cosine similarity, false negative ratio, and performance comparison on CIFAR100.}
	\label{fig:sampler_sim_fn}
\end{figure*}

\subsection{Restart Probability Analysis on CIFAR100}\label{sec:restart_cifar100}

Figure~\ref{fig:cos_sim_cifar100} shows cosine similarities and percentage of false negatives in the sampled batch among various restart probabilities $\alpha$ on CIFAR100. It can be observed that cosine similarities gradually skew left with the increase of $\alpha$ from $0.1$ to $0.7$, indicating the higher $\alpha$ brings about more hard negative instances similar to positive anchor. However, higher $\alpha$ suffers from more serious FN issues, which significantly increases the percentage of false negatives in the sampled batch. Thus, we linearly decay $\alpha$ from $0.2$ to $0.05$, achieving a trade-off between hard negatives and false negatives. More analyses are discussed in Section~\ref{sec:criterion}.

\begin{figure}[h]
    \centering
    \subfigure{
    \includegraphics[width=0.23\textwidth]{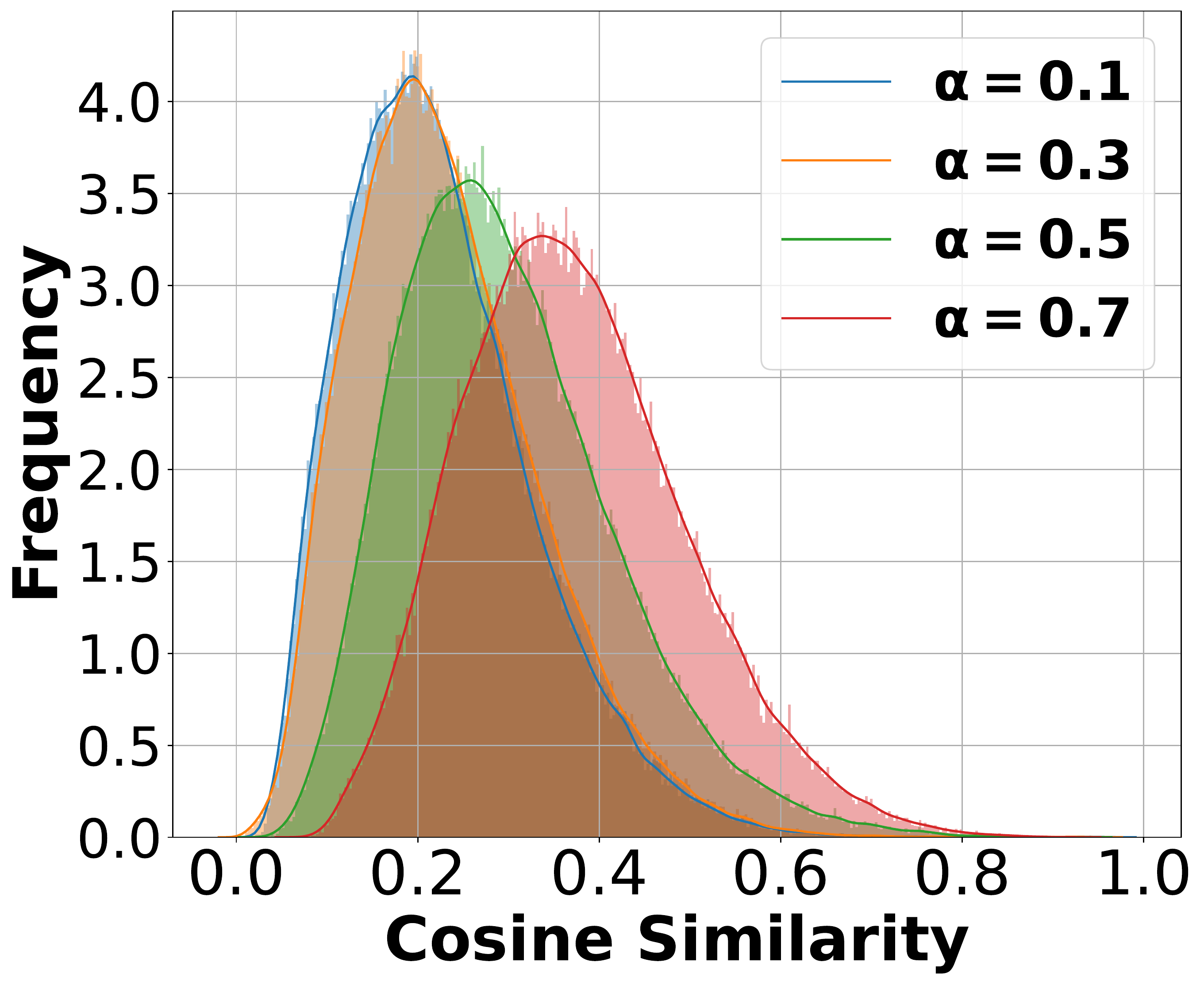}}
    \subfigure{
    \includegraphics[width=0.23\textwidth]{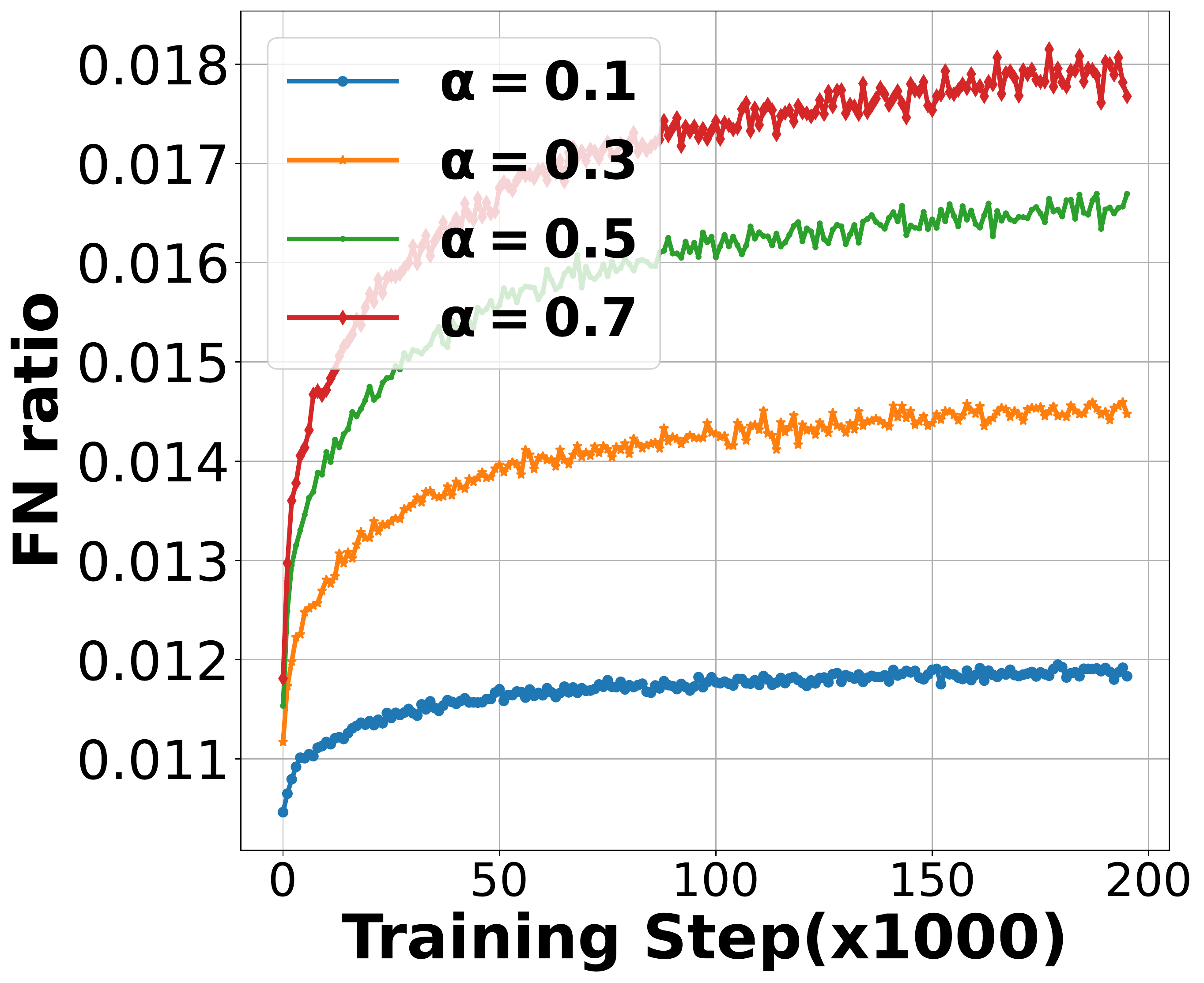}}
    \vspace{-0.3cm}
    \caption{Cosine similarities and percentage of false negatives among various restart probabilities.}
    \label{fig:cos_sim_cifar100}
    \vspace{-0.3cm}
\end{figure}

\subsection{Similarity Comparison Between Positive and Negative Pairs} \label{sec:score_comparison}

To explain the performance degradation of DCL and HCL objectives, we select 12 representative mini-batches and plot the cosine similarity histogram of positive and negative pairs on BERT (top) and RoBERTa (bottom) in Figure~\ref{fig:pos_neg_score}. 
We observe the following: (1) At the start of and throughout the training, the positive pairs are assigned a high cosine similarity~(around 0.9) by the pretrained language model; (2) The negative similarities begin with a relatively high score and gradually skew left because of the self-supervised learning. 
Such phenomenon is consistent to \citet{zhou2022debiased}. 
DCL and HCL which leverage the difference between positive and negative similarity to reweight the negative scores are inapplicable since the low distribution gap between positive and negative similarities will lead to homogeneous weighting in the objective.

\begin{figure*}[t!]
    \centering
    \includegraphics[width=\textwidth]{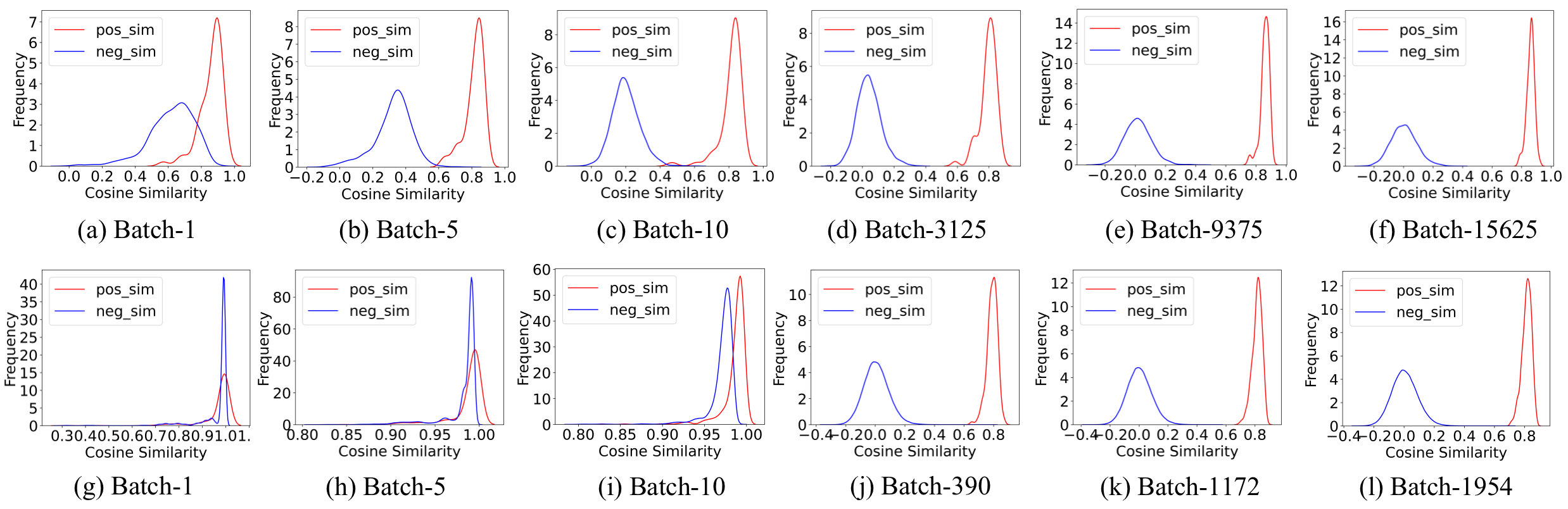}
	\caption{Histograms of cosine similarity on BERT (top) and RoBERTa (bottom).}
	\label{fig:pos_neg_score}
\end{figure*}

\subsection{Impact of Sampling Method on CIFAR100}\label{sec:sampling_cifar100}
We conduct different strategies of proximity graph sampling on CIFAR100 to investigate the impact of sampling methods. As shown in Figure~\ref{fig:graph_sampling_cifar100}, BFS achieves more similar pairs in the sampled batch compared with other proximity graph sampling methods but it introduces a higher percentage of false negatives, which obviously degrades the downstream performance. Conversely, DFS explores paths far away from the selected central node, which can not guarantee that the sampled path (i.e. batch) is within a local cluster. Thus, we theoretically leverage RWR to flexibly modulate the hardness of the sampled batch and achieve a balance between hard negatives and false negatives.

\begin{figure}[t]
    \centering
    \includegraphics[width=0.46\textwidth]{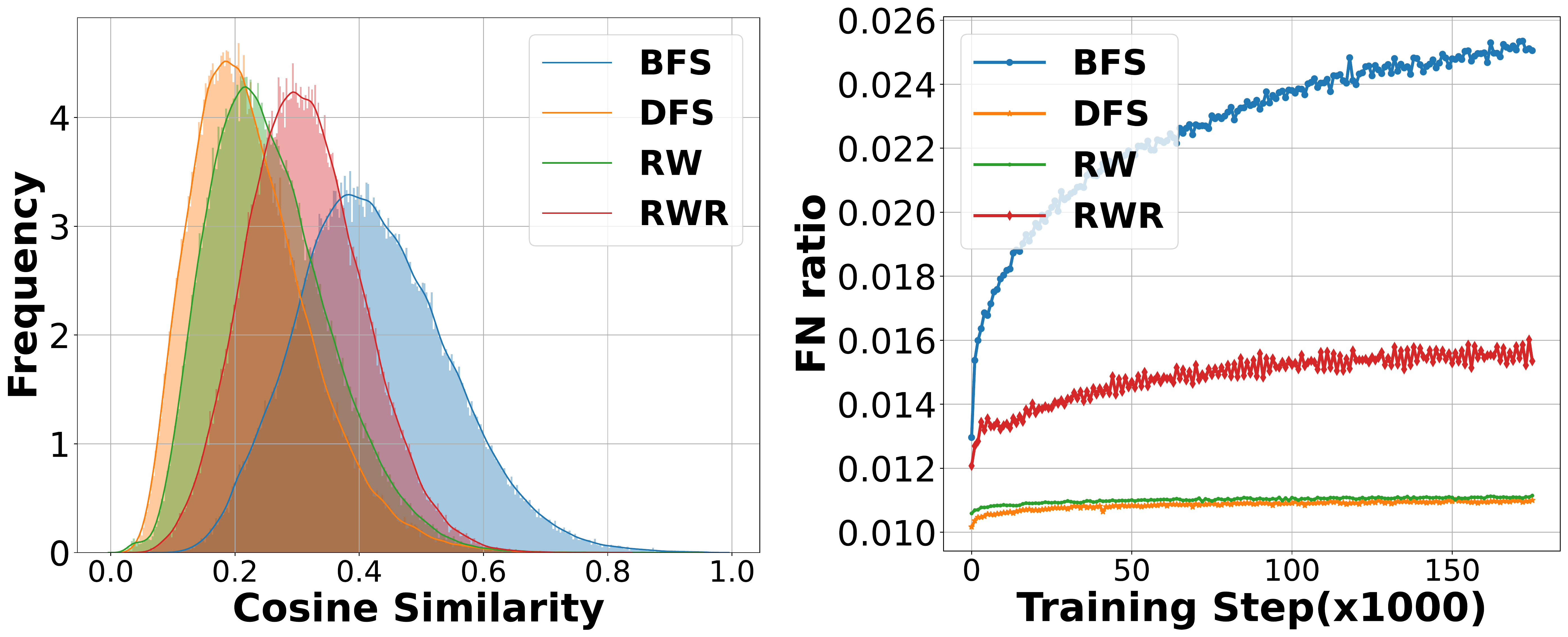}
    \vspace{-3mm}
    \caption{Histograms of cosine similarity and Percentage of false negative of all pairs in a batch for embeddings trained using different sampling methods.}    
    \label{fig:graph_sampling_cifar100}
    \vspace{-3mm}
\end{figure}

\subsection{Parameter Analysis}\label{sec:appendix_parameter}

\subsubsection{Batchsize $B$}\label{sec:appendix_N}

To analyze the impact of the batchsize $B$, we vary $B$ in the range of \{16, 32, 64, 128, 256\} and summarize the results in Table~\ref{tab:batchsize}. In vision modality, it can be found that a larger batchsize leads to better results, which is consistent with the previous studies~\citep{simclr2020,he2019momentum,mochi2020}. In language domain, \model reaches its optimum at $B=64$, which aligns with the results in SimCSE~\cite{simcse2021}. For graphs, we can observe that the performance improves slightly with increasing batch size.

\begin{small}
\begin{table}[hbpt]
    \centering
    \caption{Performance comparison of different batchsize $B$.}    \renewcommand{\arraystretch}{1.15}
    \setlength{\tabcolsep}{1.0mm}{
    \begin{tabular}{c|ccc|c|cc}
    \toprule
    \multirow{2}{*}{$B$} & \multicolumn{3}{c|}{Vision} & Language & \multicolumn{2}{c}{Graphs} \\
     & CIFAR10 & CIFAR100 & STL10 & Wikipedia & IMDB-B  & COLLAB \\
    \midrule
    16 & 79.93 & 46.69 & 56.31 & 76.26 & 71.50 & 71.32  \\
    32 & 84.64 & 56.24 & 68.61 & 74.16 & 71.60 & 71.40 \\
    64  & 89.09 & 61.30 & 74.24 & 76.69 & 71.65 &  71.42 \\
    128 & 91.03  & 65.96 & 82.56 & 76.64 &  71.83 & \textbf{71.48} \\
    256 & \textbf{92.54} & \textbf{68.68} & \textbf{84.38} & 76.11  & \textbf{71.90} & 71.35 \\
    \bottomrule
    \end{tabular}} 
    \label{tab:batchsize}
    \vspace{-8pt}
\end{table}  
\end{small}

\subsubsection{Impact of Neighbor Number $K$}\label{sec:appendix_k}

\begin{figure}[t]
    \centering    
    \includegraphics[width=0.3\textwidth]{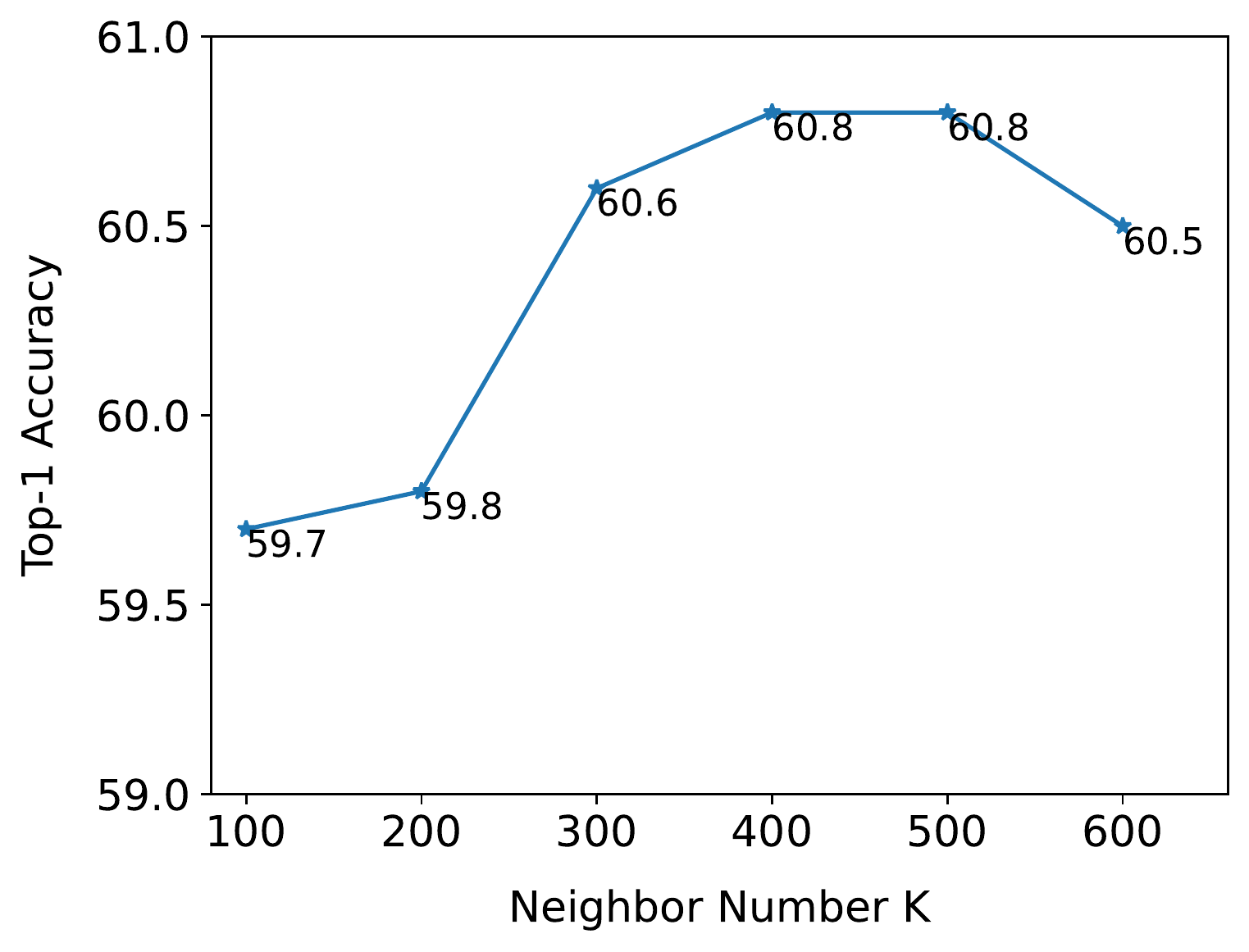}
    \vspace{-4mm}
    \caption{Impact of neighbor number $K$.}
    \label{fig:vartion_k}
  \vspace{-4mm}
\end{figure}

In Figure~\ref{fig:vartion_k}, we investigate the impact of the neighbor number $K$ on ImageNet-100 dataset with the default \model setting. We observe an absolute improvement of 1.1\%  with the increasing size of neighbors. Specifically, model achieves an absolute performance gain of 0.9\% from $K=100$ to $K=300$, while only obtaining 0.2\% from $K=300$ to $K=500$. Such experimental results are consistent with our prior philosophy, in which sampling more neighbors always increases the scale of \graph and urges \model to explore smaller-scope local clusters (i.e. sample harder negatives within a batch), leading to a significant improvement in performance at first. However, performance degrades after reaching the optimum, because larger $K$ introduces more easy negatives.

\subsubsection{Proximity Graph Update Interval $t$}
Proximity graph will be updated per $t$ training iterations, and to analyze the impact of $t$, we vary $t$ in the range of \{50,100,200,400\} for vision and \{10,25,50,100\} for graphs respectively. Table~\ref{tab:update_step} summarizes the experimental results on different update interval $t$. 
It can be observed that update intervals that are too short or too long will degrade the performance. 
The possible reason is that sampling on a \graph that is frequently updated results in unstable learning of the model. 
Besides, the distribution of instances in the embedding space will change during the training process, resulting in a shift in hard negatives. As a result, after a few iterations, the lazy-updated graph cannot adequately capture the similarity relationship.

\begin{small}
    \begin{table}[htb]
    \centering
    \caption{Performance comparison with different update interval $t$ on vision and graph modality.}
    \renewcommand{\arraystretch}{1.15}
    \setlength{\tabcolsep}{2mm}{
    \begin{tabular}{c|c|ccccc}  
    \toprule
    & Update Interval $t$     & 50 & 100 & 200 & 400  \\
    \midrule
    \multirow{2}{*}{Vision} & CIFAR10  & 92.29 & \textbf{92.54} &92.34  & 92.26    \\
     & CIFAR100  & 68.37 & \textbf{68.68} & 67.83   &   68.59  \\ 
     \hline
     \hline
     & Update Interval $t$     & 10 & 25 & 50 & 100  \\
     \hline
     \multirow{2}{*}{Graphs} & IMDB-B & 71.30 & \textbf{71.90} & 71.40 & 71.10   \\
     & COLLAB & 71.06 & 70.36 & \textbf{71.48} & 70.62 \\
    \bottomrule
    \end{tabular}} 
    \label{tab:update_step}
    \end{table}
\end{small}

\subsection{Training Curve}\label{sec:traing_curve}
We plot the training curves on STL10 and ImageNet-100 respectively. As shown in Figure~\ref{fig:converge_image}, on STL10 dataset, \model takes only about 600 epochs to achieve a similar performance as the original SimCLR, which takes 1000 epochs. 
A similar phenomenon can be seen on ImageNet-100.
All these results manifest that \model can bring model better and faster learning.

\begin{figure}[hbpt]
    \centering
    \subfigure[STL10]{
    \includegraphics[width=0.22\textwidth]{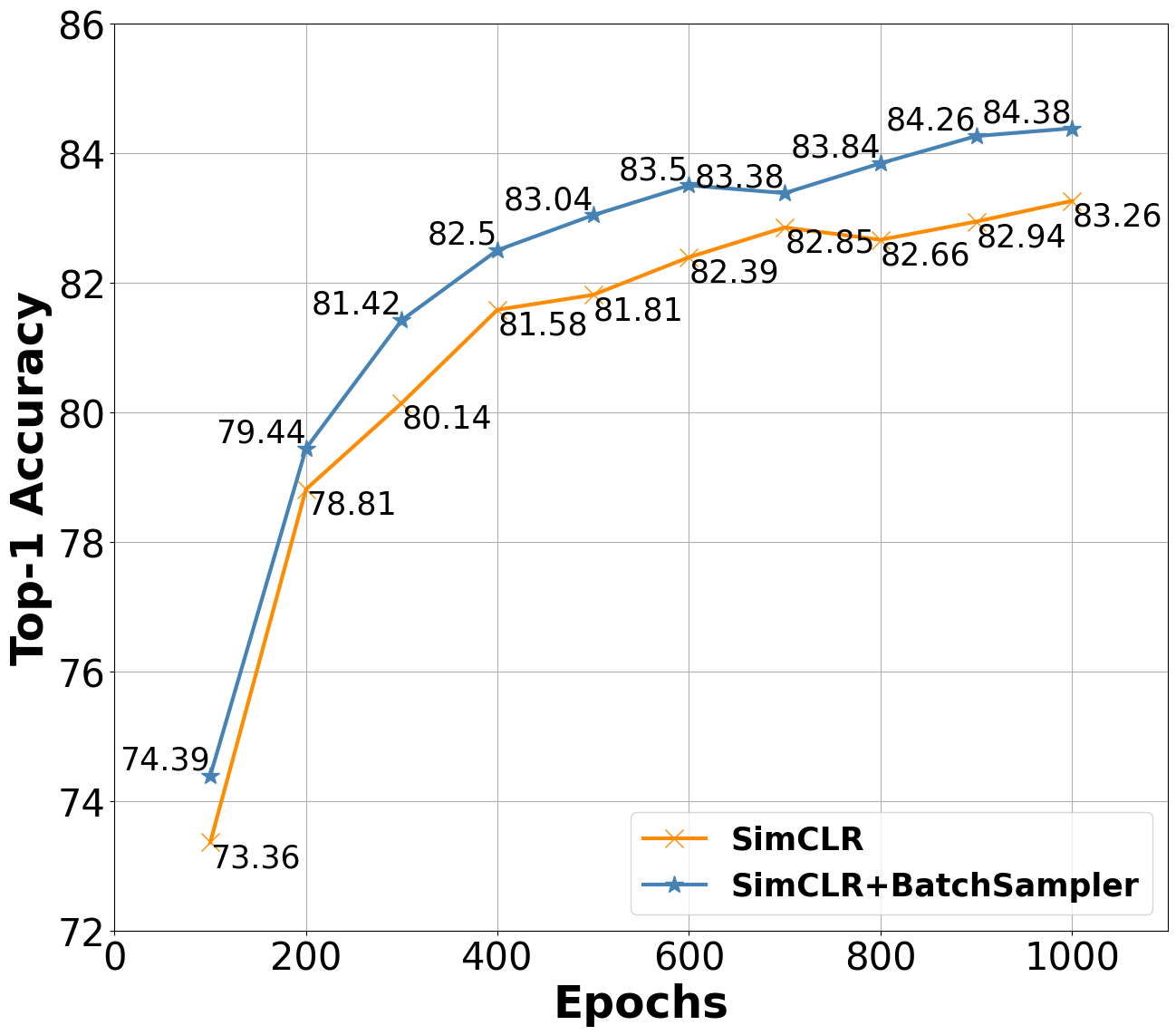}}
    \subfigure[ImageNet-100]{
    \includegraphics[width=0.22\textwidth]{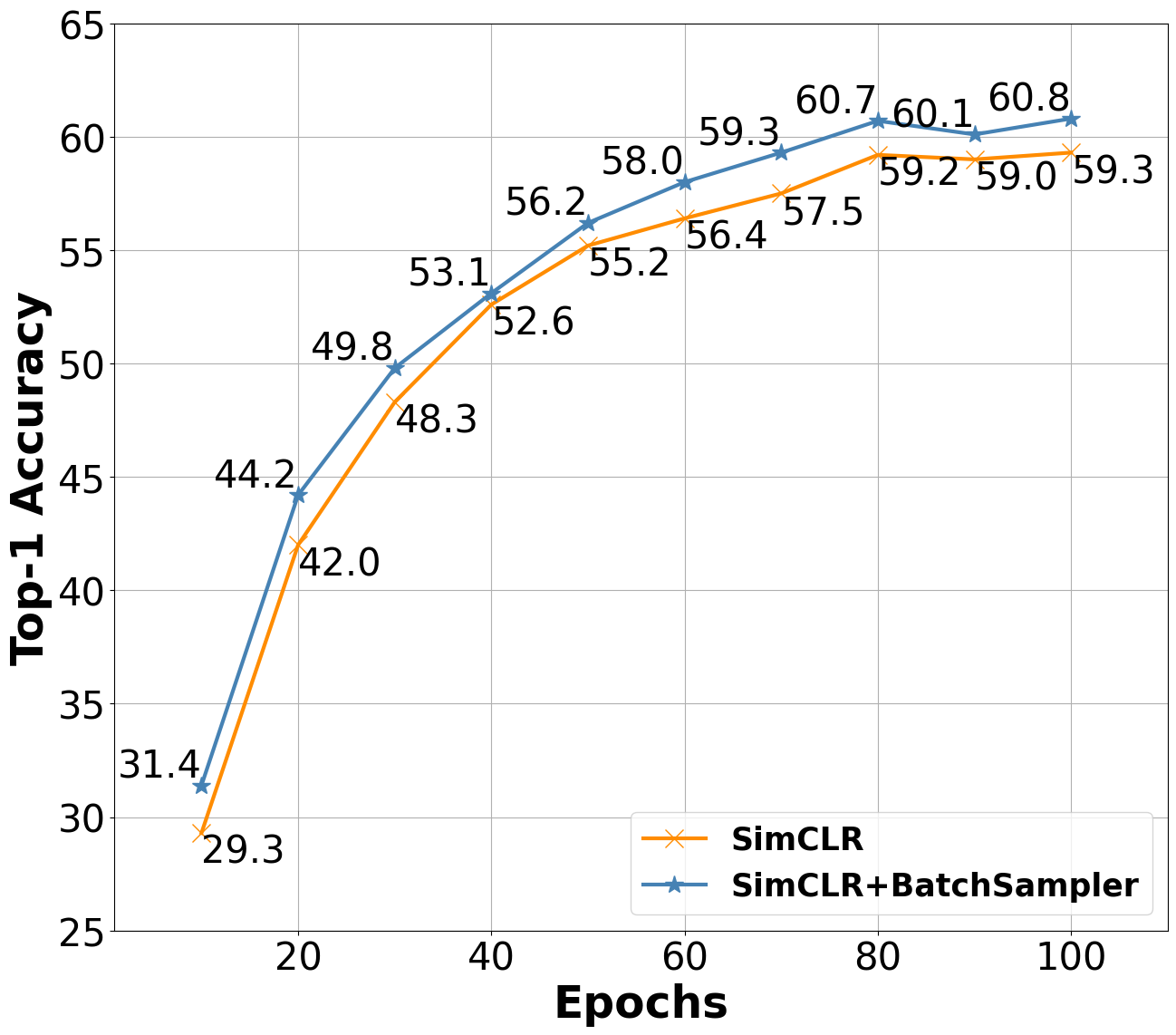}}
    \vspace{-4.mm}
    \caption{Training curves for image classification task on STL10 and ImageNet-100.}
    \label{fig:converge_image}
\end{figure}

\section{InfoNCE Objective and its Variants}\label{sec:infonce_loss}

Here we describe in detail the objective functions of three in-batch contrastive learning methods, including SimCLR~\citep{simclr2020}, GraphCL~\citep{graphcl2020} and SimCSE~\citep{simcse2021}. Besides, we cover two variants, i.e., DCL~\citep{debias2021} and HCL~\citep{hard2021}, which are also applied in the experiments.

\subsection{SimCLR}

SimCLR~\citep{simclr2020} first uniformly draws a batch of instances $\{x_1...\ x_B\}\subset\mathcal{D}$, then augments the instances by two randomly sampled augmentation strategies 
$f_{aug}(\cdot),f'_{aug}(\cdot)\sim\Set{T}$, 
resulting in $2B$ data points. 
Two augmented views $(x_i,x_{i+B})$ of the same image
are treated as a positive pair, while the other $2(B-1)$ examples are negatives. 
The objective function applied in SimCLR for a positive pair~$(x_i,x_{i+B})$ is formulated as:
\begin{equation}
\label{equ:simclr_infonce}
\ell_{i,i+B}=-\log \frac{e^{f(x_i)^T f(x_{i+B})/\tau}}{\sum_{j\neq i}^{2B}e^{f(x_i)^T f(x_j)/\tau}},
\end{equation}

where $\tau$ is the temperature and $f(\cdot)$ is the encoder. 
The loss is calculated for all positive pairs in a mini-batch, including $(x_{i},x_{i+B})$ and $(x_{i+B},x_i)$. It can be found that SimCLR takes all $2(B-1)$ augmented instances within a mini-batch as negatives.

\subsection{GraphCL and SimCSE}

Similar to SimCLR, the objective function of GraphCL~\citep{graphcl2020} and SimCSE~\citep{simcse2021} is defined on the augmented instance pairs within a mini-batch. Given a sampled mini-batch~$\{x_1...\ x_B\}\subset\mathcal{D}$, both GraphCL and SimCSE apply data augmentation to obtain positive pairs, and the loss function for a positive pair~$(x_i,x_{i}^+)$ can be formulated as:
\begin{equation}
\label{equ:simclr_graphcl}
\ell_{i}=-\log \frac{e^{f(x_i)^T f(x_{i}^+)/\tau}}{\sum_{j=1}^{B}e^{f(x_i)^T f(x_j^+)/\tau}}.
\end{equation}

Compared with the SimCLR, GraphCL and SimCSE only take the other $B-1$ augmented instances as negatives.

\subsection{DCL and HCL} 

DCL~\citep{debias2021} and HCL~\citep{hard2021} are two variants of InfoNCE objective function, which aim to alleviate the false negative issue or mine the hard negatives by reweighting the negatives in the objective.
The main idea behind them is using the positive distribution to correct for the negative distribution.

For simplicity, we annotate the positive score $e^{f(x_i)^T f(x_{i}^+)/\tau}$ as $pos$, and negative score $e^{f(x_i)^T f(x_{j}^+)/\tau}$ as $neg_{ij}$.
Given a mini-batch and a positive pair~$(x_i,x_i^+)$, the reweighting negative distribution proposed in DCL and HCL are:
\begin{equation}
\label{equ:simclr_dcl}
\max{\left(
\sum_{j=1}^{B}\frac{-N_{neg}\times\tau^+\times pos+\lambda_{ij}\times neg_{ij}}{1-\tau^+},
e^{-1/\tau}
\right)
},
\end{equation}

where $N_{neg}$ is the number of the negatives in mini-batch, $\tau^+$ is the class probability, $\tau$ is the temperature, and $\lambda_{ij}$ is concentration parameter which is simply set as 1 in DCL or calculated as $\lambda_{ij}=\frac{\beta\times neg_{ij}}{\sum neg_{ij}/N_{neg}}$ in HCL.
All of $\tau^+,\tau,\beta$ are tunable hyperparameters. 
The insight of Eq.\ref{equ:simclr_dcl} is that the negative pair with a score closer to a positive score will be assigned a lower weight in the loss function. 
In other words, the similarity difference between positive and negative pairs dominates the weighting function.

\section{Dataset Details}\label{sec:dataset_detail}

For image representation learning, we adopt five benchmark datasets, comprising CIFAR10, CIFAR100, STL10, ImageNet-100 and ImageNet ILSVRC-2012~\citep{imagent2015}. 
Information on the statistics of these datasets is summarized in
Table~\ref{tab:image_data}. 
For graph-level representation learning, we conduct experiments on 
IMDB-B, IMDB-M, COLLAB, REDDIT-B, PROTEINS, MUTAG, and NCI1~\citep{yanardag2015deep},
the details of which are presented in Table~\ref{tab:graph_data}. 
For text representation learning, we evaluate the method on a one-million English Wikipedia dataset which is used in the SimCSE and  can be downloaded from HuggingFace repository\footnote{\url{https://huggingface.co/datasets/princeton-nlp/datasets-for-simcse/resolve/main/wiki1m\_for\_simcse.txt}}.

\begin{table}[hbpt]
\centering
\caption{Statistics of datasets for image classification task.}
\renewcommand{\arraystretch}{1.15}
\setlength{\tabcolsep}{0.5mm}{
\begin{tabular}{c|ccccc}  
\toprule
 Datasets     & CIFAR10 & CIFAR100 & STL10 &  ImageNet-100 & ImageNet   \\
\midrule
  \#Train         &  50,000     & 50,000  & 105,000    & 130,000  & 1,281,167 \\
  \#Test         &  10,000      & 10,000    & 8,000  & 50,00  & 50,000\\
  \#Classes      &  10   & 100   & 10 & 100 & 1,000 \\     
\bottomrule
\end{tabular}}
\label{tab:image_data}
\vspace{-5pt}
\end{table}

\begin{small}
    \begin{table}[hbpt]
    \caption{Statistics of datasets for graph-level classification task.}
    \centering
    \renewcommand{\arraystretch}{1.15}
    \setlength{\tabcolsep}{0.2mm}{
    \begin{tabular}{c|ccccccc}  
    \toprule
     Datasets     & IMDB-B & IMDB-M & COLLAB  & REDDIT-B  & PROTEINS & MUTAG & NCI1 \\
    \midrule
      \#Graphs   & 1,000  & 1,500    & 5,000 & 2,000 & 1,113 & 188 & 4,110 \\
      \#Classes    & 2  & 3  & 3          & 2  & 2 & 2 & 2 \\
      Avg. \#nodes    & 19.8  & 13.0    & 74.5 & 429.7  & 39.1 & 17.9 & 29.8  \\      
    \bottomrule
    \end{tabular}} 
    \label{tab:graph_data}
    \end{table}
\end{small}

\begin{figure*}[t!]
    \centering
    \includegraphics[width=\textwidth]{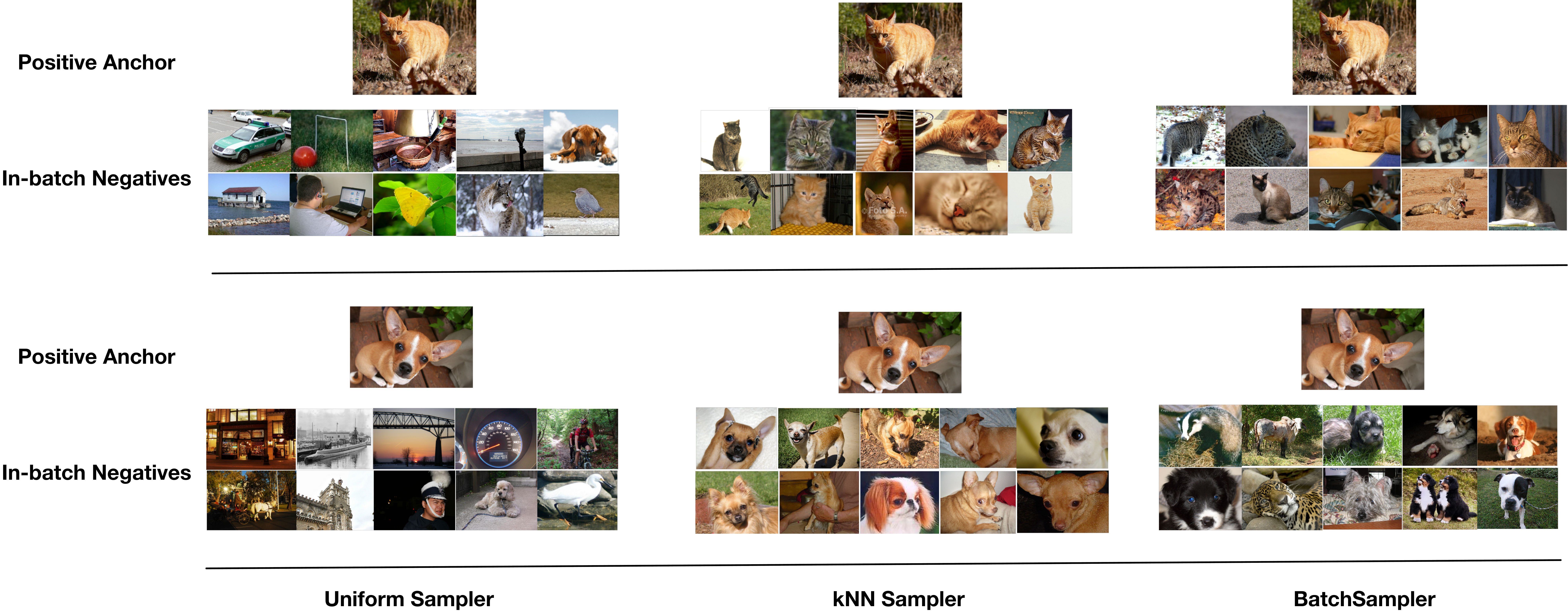}
    \caption{Case study of the negatives sampled by Uniform Sampler, kNN Sampler, and \model based on the encoder trained for 100 epochs on ImageNet. Given an anchor image~(Cat or Dog), \textbf{(Left)} in-batch negatives sampled by Uniform Sampler, \textbf{(Middle)} in-batch negatives sampled by kNN Sampler, and \textbf{(Right)} in-batch negatives sampled by \model based on \graph.}
    \label{fig:case_study}
\end{figure*}

\begin{figure*}[htbp]
    \centering
    \includegraphics[width=\textwidth]{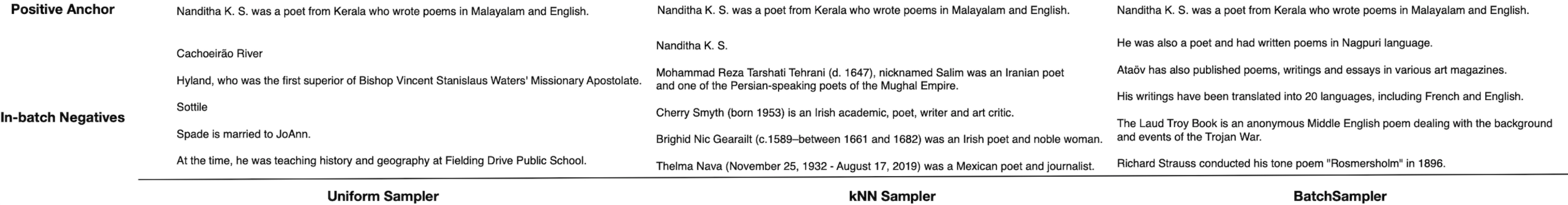}
    \caption{Case study of the negatives sampled by Uniform Sampler, kNN Sampler, and \model on language domain.}
    \label{fig:case_study_text}
\end{figure*}

\begin{figure*}[htbp]
    \centering
    \includegraphics[width=\textwidth]{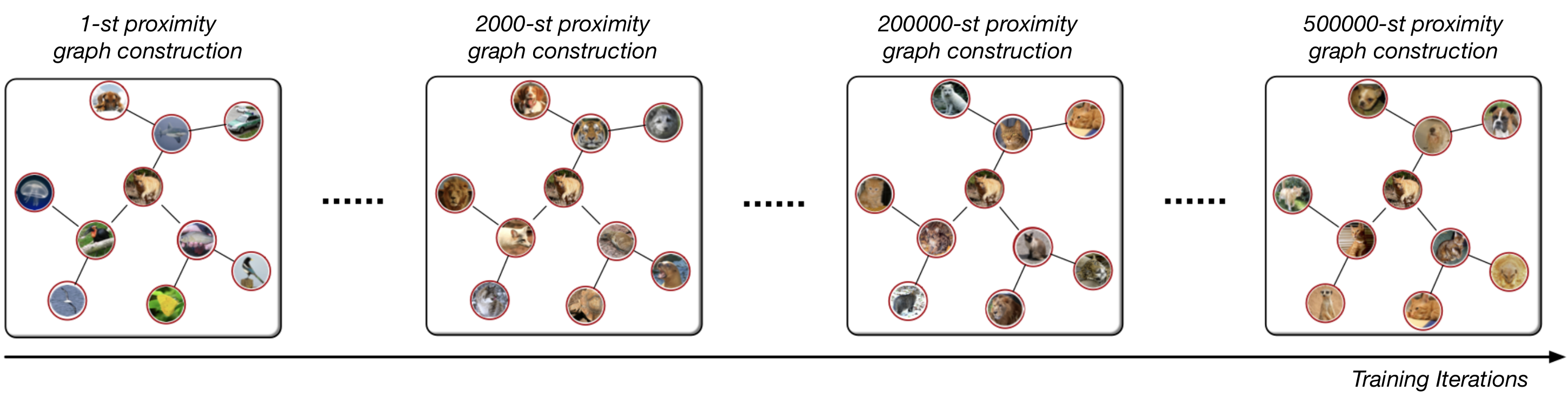}
    \caption{A visualization of the process of updating the proximity graph.}
    \label{fig:case_study_update}
\end{figure*}

\section{Experimental Details} \label{sec:experimental_settings}
\subsection{Image Representations} \label{sec:image_setup}

In image domain, we apply SimCLR~\citep{simclr2020} and MoCo v3~\citep{chen2021empirical} as the baseline method, with ResNet-50~\citep{resnet2016} as an encoder to learn image representations.
The feature map generated by ResNet-50 block is projected to a 128-D image embedding via a two-layer MLP (2048-D hidden layer with ReLU activation function). 
Besides, the output vector is normalized by $l_2$ normalization~\citep{wu2018unsupervised}.
We employ two sampled data augmentation strategies to generate positive pairs and implicitly use other examples in the same mini-batch as negative samples. 

For CIFAR10, CIFAR100 and STL10, all models are trained for 1000 epochs with the default batch size $B$ of 256. We use the Adam optimizer~\citep{kingma2014adam} with a learning rate of 0.001 for optimization. The temperature parameter is set as 0.5 and the dimension of image embedding is set as 128. 
For ImageNet-100 and ImageNet, we train the models with 100 and 400 epochs respectively, and use LARS optimizer~\citep{lars2017} with a learning rate of $0.3\times B/256$ and weight decay of $10^{-6}$. Here, the batch size is set as 2048 for ImageNet and 512 for ImageNet-100, respectively. We fix the temperature parameter as 0.1 and the image embedding dimension as 128. 
After the unsupervised learning, we train a supervised linear classifier for 100 epochs on top of the frozen learned representations.

As for \model, we update the \graph per 100 training iterations. 
We fix the number of neighbors $K$ as 100 for CIFAR10, CIFAR100 and STL10.
The size of the neighbor candidate set $M$ is set as 1000 for CIFAR100 and STL10, and 500 for CIFAR10.
Besides, the initial restart probability $\alpha$ of RWR (Random Walk with Restart) is set to 0.2 and decays linearly to 0.05 with the training process.
For ImageNet-100 and ImageNet, we keep $M$ as 1000 and $K$ as 500. The restart probability $\alpha$ is fixed as 0.1.

\subsection{Graph Representations} \label{sec:graph_setup}

In graph domain, we use the GraphCL~\citep{graphcl2020} framework as the baseline and GIN~\citep{gin2018} as the backbone.
We run \model 5 times with different random seeds and report the mean 10-fold cross-validation accuracy with variance. 
We apply Adam optimizer~\citep{kingma2014adam} with a learning rate of 0.01 and 3-layer GIN with a fixed hidden size of 32.
We set the temperature as 0.2 and gradually decay the restart probability of RWR ($0.2 \sim 0.05$). 
Proximity graph will be updated after $t$ iterations.
The overall hyperparameter settings on different datasets are summarized in Table~\ref{tab:graph_params}.

\begin{small}
    \begin{table}[hbpt]
    \caption{Hyperparameter settings for graph-level representation learning.}
    \centering
    \renewcommand{\arraystretch}{1.15}
    \setlength{\tabcolsep}{0.2mm}{
    \begin{tabular}{c|ccccccc}  
    \toprule
     Datasets     & IMDB-B & IMDB-M & COLLAB & REDDIT-B & PROTEINS & MUTAG & NCI1 \\
    \midrule
      $B$       & 256   & 128    & 128  & 128 & 128 & 128 & 128 \\
      Epoch           &  100    & 50    & 20  & 50 & 20 & 20 & 20 \\ 
      \hline
      $t$    &  25    & 25    & 50 & 50 & 50 & 50 & 50  \\
      $M$        &  500  &  500   & 1,000  & 500 & 500 & 100 & 1000 \\
      $K$       & 100  & 100 & 100 & 100 & 100 & 50 & 100 \\ 
    \bottomrule
    \end{tabular}} 
    \label{tab:graph_params}
    \end{table}
    \vspace{-5pt}
\end{small}

\subsection{Text Representations} \label{sec:text_setup}

In text domain, we use SimCSE~\citep{simcse2021} as the baseline method and adopt the pretrained BERT and RoBERTa provided by HuggingFace\footnote{\url{https://huggingface.co/models}} for sentence embedding learning. 
Following the training setting of SimCSE, we train the model for one epoch in an unsupervised manner and evaluate it on 7 STS tasks.
Proximity graph will be only built once based on the pretrained language models before training.
For BERT, we set the batch size to 64 and the learning rate to $3 \times 10^{-5}$. For RoBERTa, the batch size is set as 512 and the learning rate is fixed as $10^{-5}$.  
We keep the temperature as 0.05, the number of neighbor candidates $M$ as 1000, the number of neighbors $K$ as 500, and the restart probability $\alpha$ as 0.7 for both BERT and RoBERTa.

\subsection{Case Study}\label{sec:case_study}
To give an intuitive impression of the mini-batch sampled by \model, 
we show some real cases of the negatives sampled by Uniform Sampler, kNN Sampler, and \model in Figure~\ref{fig:case_study}. 
For a given anchor~(a cat or a dog), we apply Uniform Sampler, kNN Sampler, and \model to draw a mini-batch of images, and randomly pick 10 images from the sampled batch to show in-batch negatives for the positive anchor.    
Obviously, compared with Uniform Sampler, the images sampled by \model are more semantically relevant to the anchor in terms of texture, background, or appearance. Furthermore, kNN Sampler introduces so many false negatives that belong to the same class with positive anchor, degrading downstream performance. Thus, the proposed \model can effectively balance the exploitation of hard negatives and the FN issue. 

Furthermore, we provide a visualization of language to demonstrate the effectiveness and generalizability of \model in Figure~\ref{fig:case_study_text}. Here, we also take a vision as an example and visualize the process of updating the proximity graph in Figure~\ref{fig:case_study_update}. We maintain the default starting image and randomly sample a predetermined number of neighboring images at a specific iteration in each proximity graph.

\end{document}